\documentclass{article}

\usepackage{amsmath,amssymb,amsthm,enumerate}
\usepackage[utf8]{inputenc}
\usepackage{geometry}
\usepackage{aliascnt,bbm}

\def\Hklsimple{\mathbf{\mathcal{H}^r_{KL}}(\c,L)}
\def\Hkl{\mathbf{\mathcal{H}_{KL}^r}\left( \c n^{1+r}, n L+\Ll \right)}
\def\Hpi{\mathbf{\mathcal{H}_{\pi_0}(\Ll)}}
\def\Hn{\mathbf{\mathcal{H}_{n_0}(L,\Ll)}}
\def\Hmin{\mathbf{\mathcal{H}_{\min} } }
\def\lnbsqrt{ \sqrt{d} \log^{\beta} (n) }
\def\lnb{ d \log^{2\beta} (n) }
\def\Mnd{ d } 
\def\MndU{ n \lnb }
\def\a{\mathfrak{a}}

\def\c{\mathfrak{c}}
\def\argmin{\arg\min \,}
\def\argmax{\arg\max \,}
\def\min{\mathrm{min}}
\def\max{\mathrm{max}}
\def\1{\mathbbm{1}}

\def\rmd{\mathrm{d}}
\def\eqsp{\;} 

\def\1{\mathbbm{1}} 

\newcommand{\Ll}{\ensuremath{\ell_0}}

\newcommand\sequence[3]{\ifthenelse{\equal{#3}{}}{\ensuremath{\{
#1_{#2}\}}}{\ensuremath{\{ #1_{#2}, \eqsp #2 \in #3 \}}}} 
\newcommand\sequenceup[3] {\ifthenelse{\equal{#3}{}}{\ensuremath{\{
      #1_{#2}\}}}{\ensuremath{\{ #1^{#2}, \eqsp #2 \in #3 \}}}} 

\newcommand{\ie}{\textit{i.e.\ }}
\newcommand{\X}{\mathbf{X}}

\newcommand{\ct}{c_{3,\varepsilon}}

\newcommand{\G}{\mathcal{G}}
\newcommand{\Gt}{\mathcal{G}_t}

\newcommand{\led}{\lesssim_{u c}}
\newcommand{\ged}{\gtrsim_{u c}}

\def\ct{\aleph}

\usepackage[ruled,vlined,lined,linesnumbered]{algorithm2e}

\theoremstyle{plain}

\newtheorem{proposition}{Proposition}[section]
\newtheorem{lemma}{Lemma}[section]
\newtheorem{theorem}{Theorem}[section]
\newtheorem{definition}{Definition}[section]

\usepackage{xcolor}

\author{Marelys Crespo Navas$^{1}$, S\'{e}bastien Gadat$^{2,3}$, Xavier Gendre$^{1}$   
\\  $^{1}$ ISAE-SUPAERO, Universit\'{e} de Toulouse\\  $^{2}$Toulouse School of Economics (CNRS UMR 5314), Universit\'{e} Toulouse I Capitole \\ $^{3}$ Institut Universitaire de France}

\date{\today}

\begin{document}

\title{Stochastic Langevin Monte Carlo for (weakly) log-concave posterior distributions.}
\maketitle

\begin{abstract} 
In this paper, we investigate a continuous time version of the Stochastic Langevin Monte Carlo method, introduced in \cite{Welling-Teh}, that incorporates a stochastic sampling step inside the traditional over-damped Langevin diffusion. This method is popular in machine learning for sampling posterior distribution. We will pay specific attention in our work to the computational cost in terms of $n$ (the number of observations that produces the posterior distribution), and $d$ (the dimension of the ambient space where the parameter of interest is living). We derive our analysis in the weakly convex framework, which is parameterized with the help of the Kurdyka-\L ojasiewicz (KL) inequality, that permits to handle a vanishing curvature settings, which is far less restrictive when compared to the simple strongly convex case. We establish that the final horizon of simulation to obtain an $\varepsilon$ approximation (in terms of entropy) is of the order 
$( d \log(n)^2 )^{(1+r)^2} [\log^2(\varepsilon^{-1}) + n^2 d^{2(1+r)} \log^{4(1+r)}(n)  ]$ with a Poissonian subsampling of parameter $\left(n ( d \log^2(n))^{1+r}\right)^{-1}$, where the parameter $r$ is involved in the KL inequality and varies between $0$ (strongly convex case) and $1$ (limiting Laplace situation).
\end{abstract}


 
\footnote{
We are grateful to Patrick Cattiaux and Arnaud Guillin for helpful discussions and references on functional inequalities and especially on weak log Sobolev inequalities.
}

\section{Markovian Stochastic Langevin Dynamics and main results}

\subsection{Introduction}

\paragraph{Motivations}
In the recent past years, a huge amount of methods have been developed in machine learning to handle large scale massive datasets with a large number $n$ of observations ($X_1,\ldots,X_n)$ embedded in a high dimensional space $\mathbb{R}^d$. These methods generally involve either optimization of a data-dependent function (for frequentist learning) or sampling a data-dependent measure (for Bayesian learning with posterior distributions). In both approaches, a bottleneck lies on the size of $n$ and $d$ that usually generates numerical difficulties for the use of standard algorithms. We are interested in this paper in the simulation of a posterior distribution following a Bayesian point of view with a statistical model described by a collection of densities $(p_{\theta})_{\theta \in \Theta}$ on $\mathcal{X}$, where the parameter of interest $\theta^\star$ belongs to $\Theta= \mathbb{R}^d$ and where the $(X_i)_{1 \leq i \leq n}$ are assumed to be i.i.d. observations in $\mathcal{X}$ distributed according to $p_{\theta^\star}$. A standard Bayesian approach consists in defining a prior distribution $\pi_0$ on $\Theta$ and then sample the posterior distribution denoted by $\mu_n$ (which will be denoted by $\exp(- U_{\nu_n})$ below) using a numerical probabilistic approximation with the help of an over-damped Langevin diffusion:
$$
\rmd \theta_t = - \nabla U_{\nu_n}(t) \rmd t + \sqrt{2} \rmd B_t.
$$ In this work, we manage to deal with an adaptation of the Langevin Monte Carlo (LMC) algorithm proposed in \cite{Welling-Teh}, that exploits some old ideas of stochastic algorithms introduced in \cite{RobbinsMonro}: instead of using the previous equation, the authors propose a modification of the diffusion that generates a noisy drift in the LMC due to a sampling strategy among the set of observations $(X_i)_{1 \leq i \leq n}$. Before we provide some details on the precise objects and algorithm necessary to properly define this method, we first give some literature insights related to it.

\paragraph{State of the art}
Ergodicity and quantitative mixing properties of over-damped LMC and many other sampling algorithms is a popular subject of research initiated in the probabilistic works around, roughly speaking, two strategies. The first one relies on pathwise considerations and dynamical properties of random dynamical system and is built with some coupling argument and Lyapunov controls. We refer to the seminal contributions \cite{Meyn_Tweedie,Khasminski}, that exploits the approach of the Doeblin coupling and total variation (TV) bounds. Many extensions may be derived from this Lyapunov approach and may lead to Wasserstein or $\mathbb{L}^2$ upper bounds, we refer to \cite{cattiaux2017hitting} and the references therein of the same authors for a description of the link between Lyapunov conditions and ergodicity. The second strategy derives from spectral properties of Markov operators and is related to famous functional inequalities (Poincar\'e and Log-Sobolev among others). The general idea is to differentiate the distance along the time-evolution and apply a Gronwall Lemma to obtain a quantitative estimate of the long-time evolution of the semi-group. We refer to the seminal contributions of \cite{Holley_Stroock, BakryEmery}, and to \cite{bakry2014analysis} for an almost exhaustive survey of all possible inequalities and consequences on the ergodicity of the Markov semi-groups. Finally, let us emphasize that some strong links exist between the spectral and the Lyapunov approaches, as pointed out by \cite{CATTIAUX2008}. If functional inequalities are then strongly related to mixing properties and especially from a quantitative point of view, it is therefore necessary to develop a machinery that is able to assess these inequalities carefully, especially with a specific attention to our statistical setting of large $n$ and $d$ in the completely non-trivial situation where the target measure is \textit{log-concave} but \textit{not strongly log-concave}, which is a common feature of Bayesian posterior distributions. 

On the statistical side, the mixing properties of LMC has been largely investigated during the past decade, strongly motivated by machine learning methods such as Exponentially Weighted Aggregation introduced by \cite{Dalalyan-Tsybakov}, which involves sampling a non log-concave and heavy tailed posterior distribution. A first paper of Dalalyan \cite{Dalalyan-JRSS} establishes the cost of LMC to obtain an $\varepsilon$ TV bound in terms of $d$ and $\rho$ when the target measure is $\rho$ strongly log-concave and proposes a penalized version of LMC to circumvent the lack of strong log-concavity when the target distribution is only log-concave. Since this pioneering paper, a huge impressive literature expanded. Among others, we refer to \cite{Durmus-Moulines} that gives a careful study of discretized LMC, \cite{Dalalyan-Riou-Durand} for a kinetic version of LMC and \cite{Dalalyan-Karagulyan-Riou-Durand} where the penalized LMC in non strongly-concave situation is studied in depth. Among all these papers, first, the lack of strong log-concavity is dealt with a modification of the initial LMC using a surrogate and asymptotically vanishing penalty. Second, these papers assume that a noiseless gradient of the log-posterior is available at each iteration of the algorithm, which may not be realistic, especially with large $n$. 

Stochastic LMC (SLMC below) has attracted the interest of several works: \cite{Welling-Teh} introduced this method and described its efficiency from a numerical point of view in the particular case of Bayesian learning, which is exactly our framework. Some recent advances and related contributions may be also cited: \cite{Dalalyan-Karagulyan} studies a noisy version of LMC and derives some non-asymptotic upper bounds (in terms of Wasserstein distance) of the sampling strategy in presence of a possibly biased noise for strongly log-concave posterior distribution.
The recent contribution of \cite{Xu_2018_NIPS} is also related to our work: the authors develop a machinery for the study of SLMC essentially based on the Poincar\'e inequality but the way the lower bound on the spectral gap involved in the LMC is dealt with appears to be inappropriate. In particular, the diffusion involved in (Stochastic)-LMC is used at a very low-temperature, proportional to $1/n$, which generates some important troubles in the size of the spectral gap in non strongly log-concave framework. In \cite{Raginsky_Rakhlin}, the authors derives some close bounds to our framework for optimization purpose, and the authors identify the important dependency of the spectral gap denoted by $\lambda_*$ in their paper with the temperature level $1/\beta$ they introduced. They obtain some very highly pessimistic bounds in some general situations (see their discussion in \cite{Raginsky_Rakhlin}[Section 4]), they conclude their discussion by the urgent need to find some non-trivial situations where some better lower bounds of $\lambda_*$ may be derived. 

Indeed, the final remark of \cite{Raginsky_Rakhlin}[Section 4]) is related to the well known metastability phenomenon: at a low temperature, the mixing rates of a lot of reversible and irreversible Markov semi-groups are strongly deteriorated by the low temperature settings, which is implicitly induced by a Bayesian posterior sampling problem with a large number $n$ of observations. In a regime of variance noise of the order $O(\beta^{-1})$, the first study of large deviation principle of invariant measures traces back to \cite{Freidlin-Wentzell} where the authors establish the asymptotic of the spectral gap of the over-damped Langevin diffusion as $\exp(-I \beta)$ ( \cite{Freidlin-Wentzell}[Chapter 6]) where $I$ is an explicit constant that depends on the potential of the Gibbs field. This result has been extended in depth by \cite{Holley_Stroock}, which leads to the first precise analyses of the so-called simulated annealing method (see \textit{e.g.} \cite{Hajeck,miclo1992recuit}). These works, and more recent contributions with irreversible dynamical systems in a stochastic settings (\cite{Gadat-Panloup-Pellegrini, Gadat-Gavra}) show that there is almost nothing to expect in metastable situations in terms of asymptotic behaviour of the spectral gap, and indirectly in terms of mixing rate. Hence, the only situation that may lead to reasonable results is an intermediary situation between the (almost) trivial strongly log-concave case and the metastable multi-welled case. This is the purpose of the weakly log-concave situation that is described by the family of Kurdyka-\L ojasiewicz inequalities \cite{Kurdyka,Lojasiewicz} used in optimization theory \cite{Bolte} that have shown to be efficient for stochastic optimization \cite{Gadat-Panloup} or for sampling \cite{GPP}. We also refer to the recent contributions \cite{cattiaux2022self} that derives some functional inequalities within an intermediary framework in which the curvature $\rho$ is related to their keystone function $\alpha$ that controls the constants involved in the functional inequalities they are studying.

Taking together the statistical considerations and limitations, we are motivated in this paper in the study of the continuous time Stochastic Langevin Monte Carlo procedure. This process will be described precisely in the next paragraph as well as the Kurdyka-\L ojasiewicz setup parametrized by a real value $r$, which varies between $0$ (strongly convex case) and $1$ (limiting Laplace asymptotic tail).
We will show that the final horizon of simulation to obtain an $\varepsilon$ approximation is of the order:
$$( d \log(n)^2 )^{(1+r)^2} [\log^2(\varepsilon^{-1}) + n^2 d^{2(1+r)} \log^{4(1+r)}(n)  ]$$ with a Poissonian subsampling of parameter $\frac{1}{n ( d \log^2(n))^{1+r}}$.\\

The rest of the introduction consists in the definitions of the algorithm in Subsection \ref{subsec:algo}, the way we assess the quality of our result with an entropy criterion in Subsection \ref{subsec:entropy}, as well as the quantitative weakly log-concave assumption in Subsection \ref{subsec:KL}. We finally state our main result in Subsection \ref{subsec:mainresults}.

\subsection{Continuous time evolution \label{subsec:algo}}

Below, we briefly remind the continuous time SLMC algorithm for Bayesian learning, for which a discretized form has been introduced in \cite{Welling-Teh}. For this purpose, we consider a statistical model that is built with the help of a function $(x,\theta) \longmapsto p_\theta(x)$ where $\theta \in \mathbb{R}^d$ encodes the parameter of the statistical model and $x$ the observation in a space denoted by $\mathcal{X}$. We then assume that we have $n$ i.i.d. observations denoted by $(\X_1,\ldots,\X_n)$ distributed according to $p_\theta$. 
Given a prior distribution $\pi_0$ on $ \mathbb{R}^d$, the posterior distribution $\mu_n$ is then defined as:
$$
\mu_n(\theta) \propto \pi_0(\theta) \times \prod_{i=1}^n p_\theta(\X_i).
$$

We introduce  the log-parametrization that leads to the Gibbs form:
\begin{equation*}
U_x(\theta) = - [\log \pi_{0}(\theta) + n \log p_{\theta}(x)],
\end{equation*}
and we then observe that:
\begin{equation*}
\mu_n(\theta) \propto \exp\left(-\frac{1}{n} \sum_{i=1}^n U_{\X_i}(\theta)\right) =  \exp\left(- U_{\nu_n}(\theta)\right),
\end{equation*}
where $\nu_n$ refers to the empirical distribution and $U_{\nu_n}$ the average value of $U_X(\theta)$ when $X \sim \nu_n$:
$$
\nu_n(x) = \frac{1}{n} \sum_{i=1}^n \delta_{\X_i}(x) \quad \text{and} \quad U_{\nu_n}(\theta) = \mathbb{E}_{X \sim \nu_n}[U_X(\theta)].
$$
The standard Langevin Monte Carlo approach relies on the ergodic behaviour of the stochastic differential equation:
\begin{equation}
\label{eq:sde_Langevin}
\rmd \theta_t = -  \nabla U_{\nu_n}(\theta_t) \rmd t + \sqrt{2} \rmd B_t,
\end{equation}
that possesses under some mild assumptions a unique invariant distribution $\mu_n$. 
 
The SLMC algorithm takes benefit of both sampling with a S.D.E. and homogenization of the drift that may be written as an expectation on $X$ that is sampled uniformly over the set of observations according to $\nu_n$. The leading idea is to replace the  expectation in $U_{\nu_n}$ that depends on the overall set of observations $(\X_1,\ldots,\X_n)$  by a single unique observation that is randomized uniformly all along the evolution of the stochastic differential equation, and modified according to a Markov exponential clock. That being said, we can write an explicit formal definition of the algorithm as follows. We define $\left(\xi^{(n)}_{j}\right)_{j \ge 1}$ an infinite sequence of exponential random variables of mean $\alpha_n^{-1}$ that will be fixed later on.

We also consider a sequence $\left\{V^{(n)}_{j}\right\}_{j \ge 0}$ of i.i.d. random variables uniformly distributed in $\{1,2,\ldots,n\}$. We then define the process $(X_t)_{t \ge 0}$ as a jump process that takes its values in $\{1,2,\ldots,n\}$ such that:
\begin{equation}\label{eq:jump_X}
X_t = \left\{ \begin{array}{l c l}
\X_{V^{(n)}_{1}}, & \mbox{ if } & 0\leq t < \xi^{(n)}_{1}, \\
\\
\X_{V^{(n)}_{j}}, & \mbox{ if } & \sum\limits_{k=1}^{j-1} \xi^{(n)}_{k} \leq t < \sum\limits_{k=1}^j \xi^{(n)}_{k}, \quad j>1. 
\end{array}
\right.
\end{equation}      
Informally, $(X_t)_{t \ge 0}$ should be understood as follows: the process takes the value of one observation uniformly chosen from the $n$ observations $\X_1, \ldots, \X_n$ during exponential times with intensity $\alpha_n$.
The stochastic Langevin over-damped diffusion we consider is then given by the joint evolution $(\theta_t, X_t)_{t \ge 0}$ and that is defined by:
\begin{equation}
\label{eq:sde_theta}
\rmd \theta_t = - \nabla_{\theta} U_{X_t}(\theta_t) \rmd t + \sqrt{2} \rmd B_t, \quad t> 0,
\end{equation}
where $(B_t)_{t\ge 0}$ is a multivariate standard Brownian Motion. 

\begin{algorithm}[H] 
\caption{Stochastic Langevin over-damped}\label{ML:Algorithm1}
\SetAlgoLined
\KwData{$(\X_1,\ldots,\X_n)$ i.i.d. observations, $n_0$ initial distribution, $\pi_0$ prior distribution}
$t_0 = 0$ \\
Generate $\theta_0$ according to $n_0$ \\
\For{$k=0,1,\ldots$}{
Pick $X_k$ uniformly in $\left\{ \X_1,\ldots,\X_n \right\}$ \\
Generate $\xi_k$ according to an Exponential distribution with mean $\alpha_n^{-1}$ \\
$t_{k+1} = t_k + \xi_k$ \\
$\theta_{t_{k+1}} = \theta_{t_k} - \int_{t_k}^{t_{k+1}} \nabla_{\theta} U_{X_k} (\theta_s) \rmd s + \sqrt{2} B_{\xi_k}$ 
}
\Return $\lim\limits_{k\to \infty} \theta_{t_k}$
\end{algorithm}

\subsection{Entropic divergence \label{subsec:entropy}}

To assess the long-time behaviour of the SLMC, we introduce several notations related to the pair $(\theta_t,X_t)_{t \ge 0}$. Below, we denote by $\lambda_d$ the Lebesgue measure over $\mathbb{R}^d$. The semi-group induced by $\mathcal{L}$ being elliptic on the $\theta$ coordinate, trivially irreducible and finitely supported on the $x$ coordinate, makes the law of $(\theta_t,X_t)$ absolutely continuous with respect to the measure $\lambda_d \otimes \nu_n$ as soon as $t>0$.

We introduce the notation of $m_t$ to refer to the joint density of $(\theta_t, X_t)$ at time $t$ with respect to $\lambda_d \otimes \nu_n$. In the meantime, 
$n_t$ denotes the marginal distribution of $\theta_t$ and $m_t(\cdot|\theta)$ the conditional distribution of $X_t$ given $\theta_t=\theta$. That is:
\begin{equation}
\mbox{Law} (\theta_t, X_t) = m_t, \quad n_t(\theta) = \sum\limits_{i=1}^{n} m_t(\theta,\X_i), \quad m_t(x|\theta) = \frac{m_t(\theta,x)}{n_t(\theta)},
\end{equation}
for $\theta\in \mathbb{R}^d$ and $x\in\{\X_1,\ldots,\X_n\}$. 

To show that the SLMC algorithm recovers the correct asymptotic behaviour, \ie that $n_t (\theta) \longrightarrow \mu_n$ when $t\longrightarrow \infty$, we consider the relative entropy (or Kullback-Leibler divergence) of $n_t$ with respect to $\mu_{n}$ that is well defined thanks to the ellipticity, and given by:
\begin{equation}
\label{def:Jt}
J_t = Ent_{\mu_n}\left(\frac{n_t}{\mu_n}\right) =\int\limits_{\mathbb{R}^d} \log \left( \frac{n_t(\theta)}{\mu_n (\theta)} \right) \rmd n_t(\theta).
\end{equation}
$J_t$ measures at any time $t>0$ a divergence between the instantaneous law of the process at time $t$ and the (presumably) invariant distribution $\mu_n$ of the process $(\theta_t,X_t)$. It would also be possible to measure this difference between the two distributions in terms of the $\mathbb{L}^2$ or the $\chi$-square distance and to produce a theoretical analysis with the help of functional analysis but it would rely on stronger assumptions on the function $U_{\nu_n}$.

In the meantime, we also introduce a weighted $\mathbb{L}^2$ distance between the conditional distribution of $X_t$ given $\theta_t=\theta$ and the measure $\nu_n$. This distance is denoted by $I_t$ and is defined as:
\begin{equation}
I_t = \int_{\mathbb{R}^d} \sum\limits_{i=1}^{n} \left( \frac{m_t(\X_i|\theta)}{\nu_n (\X_i)} - 1 \right)^2 \nu_n (\X_i) \rmd n_t(\theta). \label{def_It}
\end{equation}
This quantity measures the average closeness (w.r.t. $\theta$) of the conditional law of $x$ given $\theta$ at time $t$ to $\nu_n$.

\subsection{Main assumptions \label{subsec:KL}}

\paragraph{Weak convexity}
We will study the SLMC into a weakly convex framework, \ie when $U_{\nu_n}$ is assumed to be convex but not necessarily strongly convex.
SLMC has recently received an important interest in the machine learning community and has been studied essentially in its explicit Euler discretized form in various situations where functional inequalities are involved. We refer to \cite{Wang_Osher} (uniform Log-Sobolev inequality), to \cite{Raginsky_Rakhlin} (uniform Poincar\'e inequality) where the authors develop a Wasserstein-2 analysis of the algorithm, and to \cite{Xu_2018_NIPS} (uniform Poincar\'e inequality). In these works, the functional inequalities play a crucial role to analyze the behaviour of SLMC and these inequalities are \textit{assumed}, which is an important hypothesis. Importantly, Poincar\'e or Log-Sobolev inequalities are not so innocent since they generally  require convexity (see \textit{e.g.} \cite{Bobkov_AOP_99,bakry2014analysis}) to be reasonably dimension-dependent, and even strong convexity to be dimension free. Otherwise, the constant involved in these functional inequalities are exponentially degraded by the ``temperature'' ($n^{-1}( \lnb )^{-(1+r)}$ in our case) and the dimension ($d$ for us) as indicated in \cite{Holley_Stroock}.

In our work, we have chosen to parameterize this lack of strong convexity with the help of the Kurdyka-\L ojasiewicz inequality \cite{Kurdyka,Lojasiewicz}, which is a standard tool in optimization to describe the transition between convexity and strong convexity and makes the bounds more explicit. This assumption allows to observe how the entropy evolves according to the key exponent involved in the KL inequality. In particular, it makes possible to understand the influence of the lack of strong convexity that is more or less hidden in the uniform Poincar\'e or Log-Sobolev inequalities that are assumed in the previous works. We introduce a parametric form of the KL inequalities following \cite{Gadat-Panloup}.

For this purpose, for any $V$ twice differentiable function, we denote the spectrum of the Hessian matrix of $V$ as $Sp(\nabla^2 V(\theta))$. Furthermore, if $V$ is convex, we denote:
\begin{equation*}
\underline{\lambda}_{\nabla^2V} (\theta) = \inf Sp(\nabla^2 V(\theta)).
\end{equation*}

\textbf{Hypothesis $\Hklsimple$}  
We say that a function $V:\mathbb{R}^d\to \mathbb{R}$ satisfies a $\Hklsimple$-condition if:
\begin{itemize}
\item[a)] $V$ is a $\mathcal{C}^2$-function. 
    
\item[b)] $V$ is a convex function and $\min_{\theta \in \mathbb{R}^d} V (\theta) = V(\theta^*) > 0$. 
    
\item[c)] $\nabla V$ is $L$-Lipschitz.
    
\item[d)] There exist some constants $ 0 \leq r < 1$ and $\c>0$ such that:
\begin{equation}
\c V^{-r}(\theta) \leq \underline{\lambda}_{\nabla^2 V(\theta)}  \quad \forall \theta\in \mathbb{R}^d. \label{KL inequality}
\end{equation}
\end{itemize} 
Let us briefly comment this assumption.
\begin{itemize}
\item In \cite{GPP}, a slightly different parametrization is used with the introduction of another exponent $q$ related to $\overline{\lambda}_{\nabla^2 V(\theta)} = \sup Sp(\nabla^2 V(\theta))$. The authors also assume the upper bound $\overline{\lambda}_{\nabla^2 V(\theta)} \leq \tilde{\c} V^{-q}(\theta)$. Here, we have chosen to simplify this assumption and use a rough upper bound on the eigenvalues of the Hessian matrix given by the Lipschitz constant $L$, \ie in the last inequality we simply use $\tilde{\c}=L$ and $q=0$.

\item 
We shall observe that if $V(\theta)=(1+\|\theta\|_2^2)^p$ with $p \in [1/2,1]$, then $V$ satisfies $\Hklsimple$  with $r =\frac{1-p}{p}$ and $\c=2p(1-2(1-p))$, see Remark 7 of \cite{GPP} for further details. In particular, the larger $p$, the smaller $r$, which translates into a better curvature of the potential function $V$.
\item When $r=q$, we recover a global standard KL inequality (see \cite{Gadat-Panloup,Bolte}) and when $r=1$ it corresponds to the limiting Laplace case.
\item The case $r=0$ is of course associated to the strongly convex situation where the curvature of the function is uniformly lower bounded by $\c$.    
\end{itemize}

Hence, it is expected that the complexity of SLMC increases with the lack of curvature, \ie is an increasing function of $r$. 

In section \ref{section:results on KL and U} we recall some important consequences of the KL inequality obtained in Lemma 15 of \cite{GPP}. In particular, the growth of any function that satisfies $\Hklsimple$ is lower and upper bounded by a positive power of the distance to its minimizer.

If inequality \eqref{KL inequality} holds for a constant $\c$, then it holds for all positive values less than $\c$. For that reason, in section \ref{section:result semigroup} we assume $\c \leq \left( \frac{8L}{(1+r)} \right)^{1+r}$.

\paragraph{Assumption on the prior $\pi_0$}

We state below the important consequence of a ``population'' $\Hklsimple$ assumption, but before, let us state some mild assumptions on $\pi_0$. \vspace{.08cm}

\textbf{Hypothesis $\Hpi$} $\pi_0$ is a log-concave $\mathcal{C}^2$-function such that $\min_{\theta \in \mathbb{R}^d} - \log \pi_0(\theta) > 0$ and $\theta \mapsto \nabla \log \pi_0(\theta)$ is $\Ll$-Lipschitz. \vspace{.2cm}

Since the prior distribution is chosen by the user, our $\Hpi$ hypothesis is not restrictive and some typical examples satisfy these conditions, such as Gaussian, Weibull and Gamma, both with shape parameter larger than 1, Gumbel, among others.

\begin{proposition}\label{prop:kl_to_kl} We assume $\Hpi$ and that there exist $(\c,r)$ such that for any $x$: $\theta \longmapsto - \log p_\theta(x)$ satisfies $\Hklsimple$, then $U_{\nu_n}$ satisfies $\Hkl$, and in particular, for any $\X_i$, $U_{\X_i}$ satisfies $\Hkl$.
\end{proposition}

We introduce the notation $a \led b$ ($a \ged b$) which means $a \leq c b$ ($a \geq c b$) where $c$ is a universal constant \ie a positive constant independent of $n$ and $d$.

We assume that the minimizers of the functions $U_{\X_i}$ are contained in a ball of radius which depends of $n$ and $d$. Additionally, we consider $\min_{\theta\in\mathbb{R}^d}U_{\X_i}$ to be at most of order $\Mnd$.  \vspace{.2cm}

\textbf{Hypothesis $\Hmin$} There exists $\beta \geq 0$ such that:
$$
\max_{i} \|\argmin U_{\X_i}\|_2 \led \lnbsqrt \quad \mbox{ and } \quad \max_i \ \min_{\theta \in \mathbb{R}^d} \ U_{\X_i} (\theta) \led \Mnd. 
$$ 

Assumption $\Hmin$ is not restrictive. In dimension $d=1$, it holds for many concentrated i.i.d. samples $(X_i)_{1 \leq i \leq n}$ with a suitable sub-Gaussian like behaviour for which the Laplace transform of $ \min \ U_{X_i}$ is upper bounded as:
$$
\mathbb{E}[\exp(\lambda \min \ U_{X_i} )] \leq \exp(\sigma^2 \lambda^{k}), \quad \forall \lambda >0.
$$
The previous upper bound implies that, in this case, $\beta$ involved in $\Hmin$ is given by $\beta=\frac{k-1}{k}$. We recover in particular the situation where $\beta=1/2$ when $k=2$. For larger dimensions, the result may be extended using that $\|X\|_2^2 \leq d \, \max_{1\leq j \leq d} (X^j)^2$, where $X^j$ is the $j$-th component of $X$.
We should keep in mind from this last discussion that even if $\Hmin$ is stated (and makes sense) for any value of $\beta>0$, it holds in general for $\beta \leq 1$.

This $\Hmin$ hypothesis together with $\Hpi$ lead to an almost similar behaviour of the minimizer and the minimum of $U_{\nu_n}$. Details appear in Proposition \ref{prop:growth_min_KL}.  

\subsection{Long-time entropy convergence\label{subsec:mainresults}}

We introduce for any time $t \ge 0$ the density of $\mbox{Law} (\theta_t)$ w.r.t. $\mu_n$, which is given by:
$$
f_t(\theta) = \frac{n_t(\theta)}{\mu_n(\theta)},
$$
and $n_0$ is chosen such that $\|f_0\|_{\infty} < + \infty$. The following hypothesis guarantees this result which will be proved in Proposition \ref{prop:oscillation_borne}.  \vspace{.2cm}

\textbf{Hypothesis $\Hn$} 
A positive constant $\sigma^2$ exists such that $n_0 = \mathcal{N}(0,\sigma^2 I_d)$. Moreover, there exist two universal constants $c_1$ and $c_2$ such that $0<c_1\leq c_2 < 1$ and 
$$\frac{c_1}{n L + \Ll} \leq \sigma^2 \leq \frac{c_2}{n L + \Ll}.$$

Futhermore, in Proposition \ref{prop:oscillation_borne}, as an immediate consequence of the boundedness of $\|f_0\|_{\infty}$, we obtain that $J_0 \led n d^{1+r} \log^{2\beta(1+r)} (n) + d \log \left(\frac{d }{ n}\right)$.

The next result assesses a mixing property in terms of decrease of the entropy and therefore states the convergence of $n_t$ towards the correct measure $\mu_n$.

\begin{theorem}\label{theo:convergence}
Assume $\Hpi$, $\Hmin$, $\Hn$ and that each $\theta \mapsto - \log p_\theta(\X_i)$ satisfies $\Hklsimple$, then
\begin{itemize}
\item $U_{\nu_n}$ satisfies a Poincar\'e inequality of constant $C_P(\mu_n)$, indistinctly denoted as $C_P$. 
\item Define $c_{n,d}:=n^{4} \left( \lnb \right)^{1+r}$ and $O_{n,d} := \left(\frac{C_1 d }{ n}\right)^{\frac{d r }{2}} \exp\left(C_2 n \left( \lnb \right)^{1+r} \right)$, where $C_1$ and $C_2$ are universal constants, then for any $t>0$:
\begin{equation}\label{eq:borne_Jt}
J_t \led \left( J_0 + \frac{c_{n,d}}{\alpha_n } \left[1+\left(\frac{C_P}{\alpha_n}+\sqrt{C_P}\right) e^{\frac{\sqrt{C_P}}{\sqrt \a} + \frac{C_P}{3\alpha_n}}\right] + O_{n,d} \right) (1+t)^{1/4} e^{-\frac{\sqrt{C_p}}{\sqrt{a}} (\sqrt{1+t}-1)}.
\end{equation}
\item For any $\varepsilon >0$, if $\alpha_n = \frac{1}{n\left( \lnb \right)^{1+r}}$, then:
$$
t \ged \left( \lnb \right)^{(1+r)^2} \left[\log^2(\varepsilon^{-1}) + n^2 \left( \lnb \right)^{2(1+r)} + d^2 \log^2 d \right] \Longrightarrow J_t \leq \varepsilon.
$$
\end{itemize}
\end{theorem}
If we denote $t_\varepsilon$ the smallest value such that $J_{t_\varepsilon}\leq \varepsilon$, then the choice of $\alpha_n = \frac{1}{n\left( \lnb \right)^{1+r}}$ guarantees that the mean number of jumps $\alpha_n t_{\varepsilon}$ of the process $(X_t)_{0\leq t \leq t_\varepsilon}$ is the minimum possible. 

In order to proof the main result, we first present in Section \ref{section:markov tools} the classical tools related to the Markov semi-group, which could be skipped by the experienced reader in the subject. In Section \ref{section:Proof main results} we prove the main result. Sections \ref{section:results on KL and U} and \ref{section:result semigroup} are reserved to the technical results of the $\Hklsimple$ hypothesis and $U_{\nu_n}$, and the Markov Dynamics respectively.

\section{Markov tools \label{section:markov tools}}
It is straightforward to verify that the joint evolution of ($\theta_t,X_t)_{t \ge 0}$ exists and is weakly unique (in law) with the help of the Martingale Problem (MP below). For this purpose, we preliminary define the operator $\mathcal{L}$ that acts on any function $f \in \mathcal{C}^2(\mathbb{R}^d \times \mathcal{X})$ as:
\begin{equation}\label{def:L}
\mathcal{L}f(\theta,x) = \underbrace{- \langle \nabla_{\theta} U_x(\theta),\nabla_{\theta} f(\theta,x)\rangle + \Delta_\theta f(\theta,x)}_{:=\mathcal{L}_1 f(\theta,x)}  + \underbrace{\frac{\alpha_n}{n} \sum_{i=1}^n [ f(\theta,\X_i)-f(\theta,x)}_{\mathcal{L}_2 f(\theta,x)} ],
\end{equation}
for all $(\theta,x) \in \mathbb{R}^d \times \mathcal{X}$.

The operator $\mathcal{L}$ is divided into two terms, $\mathcal{L}_1$ acts on the component $\theta$ and is associated to the diffusion part, while $\mathcal{L}_2$ is the jump operator that acts on the $x$ component.
Thanks to the finiteness of the number of observations $(\X_1,\ldots,\X_n)$, we can apply the results of 
Section 4 and 5 of chapter 4 of \cite{Ethier_Kurtz} and deduce the following result:
\begin{proposition}
Assume that for any $x\in \mathcal{X}$, $U_x$ is $\mathcal{C}^2(\mathbb{R}^d)$ and $\nabla_\theta U_x$ is $L_x$-Lipschitz, then for any initial distribution $\nu$ on $\mathbb{R}^d \times \mathcal{X}$, the martingale problem $(\mathcal{L},\nu)$ is well-posed. 

The associated (weakly) unique process $(\theta_t,X_t)_{t \ge 0}$ is a Feller Markov process associated to the semi-group $\mathcal{L}$. In particular, the $\theta$ component verifies the S.D.E. \eqref{eq:sde_theta}.
\end{proposition}

If we denote by 
$\mathcal{L}^\star$ the adjoint operator of $\mathcal{L}$ in $\mathbb{L}^2(\mathbb{R}^d) \times \nu_n$, the backward Kolmogorov Equation  yields:
\begin{equation}\label{eq:kolmogorov}
\partial_t m_t(\theta,x) = \mathcal{L}^\star m_t(\theta,x).
\end{equation}

Using the ellipticity of the semi-group $\mathcal{L}$ on the $\theta$ coordinate, we can use the result of \cite{hormander1967hypoelliptic} and deduce that for any $t >0$, $n_t \in \mathcal{C}^{\infty}(\mathbb{R}^d,\mathbb{R})$ and the irreducibility yields 
$\forall t \ge 0, n_t >0$. We will prove in Proposition \ref{prop:oscillation_borne} some sufficient conditions that implies $\|f_0\|_{\infty} = \|\frac{n_0(\theta)}{\mu_n(\theta)} \|_{\infty} < + \infty$ and an important and standard consequence of the maximum principle, is as follows: if $\|f_0\|_{\infty} \leq M$, then
$$
\forall t  \ge 0, \qquad \|f_t\|_{\infty} \leq M.
$$
We defer the details of this result to the Proposition \ref{prop:oscillation_borne} as they are not central to our analysis and are rather technical.

Thanks to this master equation, it is possible to compute the derivative of the semi-group on some time dependent function of $\theta$. For this purpose, we introduce two keystone operators. The first one describes the infinitesimal action on the $\theta$ coordinate under the average effect of $X_t$ at time $t$ that applies $\forall f \in \mathcal{C}^2(\mathbb{R}^d,\mathbb{R})$ as:
\begin{equation}\label{eq:Gt}
\Gt  f(\theta) = - \sum\limits_{i=1}^n \langle\nabla_{\theta} f(\theta), \nabla_{\theta} U_{\X_i}(\theta)\rangle m_t(\X_i|\theta) + \Delta_{\theta} f(\theta).
\end{equation}
The second one is very close to the first one except that the average effect of $X_t$ is replaced by the targeted ideal distribution $\nu_n$. It leads to the definition $\forall f \in \mathcal{C}^2(\mathbb{R}^d,\mathbb{R})$:
\begin{equation}\label{eq:G}
\G  f(\theta) = - \sum\limits_{i=1}^n \langle\nabla_{\theta} f(\theta), \nabla_{\theta} U_{\X_i}(\theta)\rangle \nu_n(\X_i) + \Delta_{\theta} f(\theta) = 
 -  \langle\nabla_{\theta} f(\theta), \nabla_{\theta} U_{\nu_n}(\theta)\rangle + \Delta_{\theta} f(\theta). 
\end{equation}

This derivative is given in the next result, whose proof is deferred to the appendix.
\begin{lemma}\label{lemma:derivation_semi_group}
Let be $h_t$ a twice differentiable function with uniformly bounded first and second order derivatives on $\mathbb{R}^d$, then for $t>0$:
\begin{equation}
\partial_t \left\{ \int_{\mathbb{R}^d} h_t(\theta) \rmd n_t(\theta)\right\} = \int_{\mathbb{R}^d} \partial_t\{ h_t(\theta)\} \rmd n_t(\theta) + \int_{\mathbb{R}^d} \Gt h_t(\theta) \rmd n_t(\theta), \label{eq:op_theta} 
\end{equation}
where $\Gt$ is the diffusion operator under the average effect of $X_t$, defined in Equation \eqref{eq:Gt}.
\end{lemma}

\section{Proof of the main results \label{section:Proof main results}}
\subsection{Evolution of the entropy $J_t$}

The entropy satisfies the following differential inequality.
\begin{proposition}\label{prop:diff_Jt}
Assume $\Hmin$, $\Hpi$ and for each $\X_i$, $\theta \to - \log p_{\theta}(\X_i)$ satisfies $\Hklsimple$, then a "universal" constant $C$ (independent from $n$ and $d$) exists such that $\forall t > 0$: 
$$
\partial_t \{J_t\} \leq 
- \int_{\mathbb{R}^d} \left\|\nabla_{\theta} \left( \sqrt{\frac{n_t(\theta)}{\mu_{n} (\theta)}} \right)\right\|_2^2 \rmd \mu_{n}(\theta) + C I_t^{\frac{1}{3}} n^{\frac{11}{3}} \left( \lnb \right)^{1+r}.
$$
\end{proposition}

\begin{proof}
We shall use the standard preliminary estimate that may be derived from Equation (3.14) of \cite{KUSUOKA1984271} for elliptic diffusions to apply Lemma \ref{lemma:derivation_semi_group} to $f_t=\log(n_t \mu_n^{-1})$.  From Equation \eqref{eq:op_theta}, we have: 
\begin{equation}
\partial_t\{J_t\} = \int_{\mathbb{R}^d} \partial_t \left\{ \log \left( \frac{n_t(\theta)}{\mu_{n} (\theta)} \right)\right\} \rmd n_t(\theta) + \int_{\mathbb{R}^d} \Gt \log \left( \frac{n_t(\theta)}{\mu_{n} (\theta)} \right) \rmd  n_t(\theta),  \nonumber  
\end{equation}
The first term vanishes since:
\begin{eqnarray*}
\int_{\mathbb{R}^d} \partial_t \left\{ \log \left( \frac{n_t(\theta)}{\mu_{n} (\theta)} \right) \right\} \rmd n_t(\theta) &=& \int_{\mathbb{R}^d}  \frac{\partial_t \{ n_t(\theta)\}}{n_t(\theta)} \rmd n_t(\theta) \nonumber \\
&=& \int_{\mathbb{R}^d} \partial_t \left\{ n_t(\theta)\right\} \rmd \theta \nonumber \\
&=& \partial_t \left( \int_{\mathbb{R}^d} \rmd n_t(\theta) \right) \nonumber \\
&=& 0. \nonumber
\end{eqnarray*}

Then, the derivative is reduced to the second term, and we are led to:
\begin{eqnarray}
\partial_t \{ J_t \}  &=& \int_{\mathbb{R}^d} \Gt \log \left( \frac{n_t(\theta)}{\mu_{n} (\theta)} \right) \rmd n_t(\theta),  \nonumber \\
&=& \underbrace{ \int_{\mathbb{R}^d} \G \log \left( \frac{n_t(\theta)}{\mu_{n} (\theta)} \right) \rmd n_t(\theta) }_{J_{1,t}} + \underbrace{\int_{\mathbb{R}^d} \left( \Gt  - \G \right) \log \left( \frac{n_t(\theta)}{\mu_{n} (\theta)} \right) \rmd n_t(\theta)}_{J_{2,t}}. 
\end{eqnarray}	
We study the two terms $J_{1,t}$ and $J_{2,t}$ separately.	
\begin{itemize}
\item Study of $J_{1,t}$.   
Since $\G$ is a diffusion operator and $\mu_{n}$ is the invariant measure associated to $\G$, then we can use the classical link between $J_{1,t}$ and the Dirichlet form (see \cite{bakry2014analysis}):
\begin{eqnarray}
\int_{\mathbb{R}^d} \G \log \left( \frac{n_t(\theta)}{\mu_{n} (\theta)} \right)  \rmd n_t(\theta) &=& \int_{\mathbb{R}^d} \frac{n_t(\theta)}{\mu_{n} (\theta)} \ \G \log \left( \frac{n_t(\theta)}{\mu_{n} (\theta)} \right)  \rmd \mu_{n} (\theta) \nonumber \\ 
&=& - 4 \int_{\mathbb{R}^d} \left\| \nabla_{\theta} \left( \sqrt{ \frac{n_t(\theta)}{\mu_{n} (\theta)} } \right) \right\|_2^2 \rmd \mu_{n} (\theta).\label{ineq:form_dirichlet}
\end{eqnarray}
 
\item Study of $J_{2,t}$. 
We use the difference between $\G$ and $\Gt$,  for any twice differentiable function $f$:
\begin{eqnarray*}
\left( \Gt - \G \right) f(\theta) &=&  - \sum\limits_{i=1}^n \left\langle\nabla_{\theta} f(\theta), \nabla_{\theta} U_{\X_i}(\theta)\right\rangle \left[ m_t(\X_i|\theta) - \nu_n (\X_i) \right] \nonumber \\
&=& -  \sum\limits_{i=1}^n \langle\nabla_{\theta} f(\theta), \nabla_{\theta} U_{\X_i}(\theta)\rangle \left[ \frac{m_t(\X_i|\theta)}{\nu_n (\X_i)}-1 \right] \nu_n(\X_i). 
\end{eqnarray*}
Then, the term $J_{2,t}$ may be computed as:
\begin{eqnarray*}
\left| J_{2,t} \right| &=& \left| \int_{\mathbb{R}^d} \left( \Gt - \G \right) \log \left( \frac{n_t(\theta)}{\mu_{n} (\theta)} \right) \rmd n_t(\theta) \right| \nonumber \\
&=&  \left| \int_{\mathbb{R}^d} \sum\limits_{i=1}^n \langle\nabla_{\theta} \log \left( \frac{n_t(\theta)}{\mu_{n} (\theta)} \right), \nabla_{\theta} U_{\X_i}(\theta)\rangle \left[ \frac{m_t(\X_i|\theta)}{\nu_n (\X_i)}-1 \right] \nu_n(\X_i) \ \rmd n_t(\theta) \right|. \nonumber
\end{eqnarray*}
Using the Cauchy-Schwartz inequality with respect to the measure $\nu_n(\X_i)\times \rmd n_t(\theta)$ in the first line, $2 ab \le a^2+b^2$ in the second line and $\nabla \log f = 2 \nabla \log \sqrt{f} = 2 \frac{\nabla \sqrt{f}}{\sqrt{f}}$ in the third line, we obtain that:{ \footnotesize
\begin{eqnarray}
\left| J_{2,t} \right| &\leq&  \left( \int_{\mathbb{R}^d} \left\|\nabla_{\theta} \log \left( \frac{n_t(\theta)}{\mu_{n} (\theta)} \right) \right\|_2^2 \rmd n_t (\theta) \right)^{\frac{1}{2}} \left( \int_{\mathbb{R}^d} \sum\limits_{i=1}^n \left\| \nabla_{\theta} U_{\X_i}(\theta) \right\|_2^2 \left[ \frac{m_t(\X_i|\theta)}{\nu_n (\X_i)}-1 \right]^2 \nu_n(\X_i) \ \rmd n_t(\theta) \right)^{\frac{1}{2}} \nonumber \\
&\leq& \frac{3}{4} \int_{\mathbb{R}^d} \left\|\nabla_{\theta} \log \left( \frac{n_t(\theta)}{\mu_{n} (\theta)} \right) \right\|_2^2 \rmd n_t (\theta) +  \frac{1}{3} \int_{\mathbb{R}^d} \sum\limits_{i=1}^n \left\| \nabla_{\theta} U_{\X_i}(\theta) \right\|_2^2 \left[ \frac{m_t(\X_i|\theta)}{\nu_n (\X_i)}-1 \right]^2 \nu_n(\X_i) \  \rmd n_t(\theta) \nonumber \\
&\leq& 3 \int_{\mathbb{R}^d} \left\|\nabla_{\theta} \left( \sqrt{\frac{n_t(\theta)}{\mu_{n} (\theta)}} \right)\right\|_2^2 \rmd \mu_{n}(\theta) + \frac{1}{3} \int_{\mathbb{R}^d} \sum\limits_{i=1}^n \left\| \nabla_{\theta} U_{\X_i}(\theta) \right\|_2^2 \left[ \frac{m_t(\X_i|\theta)}{\nu_n (\X_i)}-1 \right]^2  \nu_n(\X_i) \ \rmd n_t(\theta). \nonumber
\end{eqnarray}}
\end{itemize}
Using Equation \eqref{ineq:form_dirichlet} and the previous line yields:
{\footnotesize
\begin{eqnarray}
\partial_t\{ J_t\}&\leq& - \int_{\mathbb{R}^d} \left\|\nabla_{\theta} \left( \sqrt{\frac{n_t(\theta)}{\mu_{n} (\theta)}} \right)\right\|_2^2 \rmd \mu_{n}(\theta) + \frac{1}{3} \underbrace{ \int_{\mathbb{R}^d} \sum\limits_{i=1}^n \left\| \nabla_{\theta} U_{\X_i}(\theta) \right\|_2^2 \left[ \frac{m_t(\X_i|\theta)}{\nu_n (\X_i)}-1 \right]^2  \nu_n(\X_i) \ \rmd n_t(\theta)}_{:=\Delta_t}, 
\end{eqnarray}}

We then focus on the second term of the right hand side. For this purpose, we consider a non-negative function $g(t)$, which will be fixed later and we split $\Delta_t$ into two terms as: {\footnotesize
\begin{eqnarray*}
\Delta_t &=& \int_{\mathbb{R}^d} \sum\limits_{i=1}^n \left\| \nabla_{\theta} U_{\X_i}(\theta) \right\|_2^2 \left(\1_{\|\nabla_{\theta} U_{\X_i}(\theta)\|_2\leq g(t)} + \ \1_{\|\nabla_{\theta} U_{\X_i}(\theta)\|_2 > g(t)} \right) \left[ \frac{m_t(\X_i|\theta)}{\nu_n (\X_i)}-1 \right]^2  \nu_n(\X_i) \ \rmd n_t(\theta) \nonumber \\
&\leq& g^2(t) I_t 
+ \int_{\mathbb{R}^d} \sum\limits_{i=1}^n \left\| \nabla_{\theta} U_{\X_i}(\theta) \right\|_2^2 \1_{\|\nabla_{\theta} U_{\X_i}(\theta)\|_2 > g(t)} \left[ \frac{m_t(\X_i|\theta)}{\nu_n (\X_i)}-1 \right]^2 \nu_n(\X_i) \ \rmd n_t(\theta), \nonumber
\end{eqnarray*}}
where $I_t$ has been introduced in Equation \eqref{def_It} and measures the closeness of $m_t(\X_i | \theta)$ to $\nu_n$.
Finally, for the last term we observe that $0\leq m_t(\X_i|\theta) \leq 1$ and $\left|\frac{m_t(\X_i|\theta)}{\nu_n (\X_i)}-1\right| = n \left| m_t(\X_i|\theta)-\frac{1}{n} \right| \leq n$, which implies that:
\begin{equation}\label{eq:Majo_Delta_t}
\Delta_t \leq g^2(t) I_t + n^2 \underbrace{ \frac{1}{n} \int_{\mathbb{R}^d} \sum\limits_{i=1}^n \left\| \nabla_{\theta} U_{\X_i}(\theta) \right\|_2^2 \1_{\|\nabla_{\theta} U_{\X_i}(\theta)\|_2 > g(t)} \rmd n_t(\theta)}_{:=\tilde{\Delta}_t}. 
\end{equation}

\noindent
The Cauchy inequality leads to:
\begin{eqnarray}
\tilde{\Delta}_t &\leq& \left( \frac{1}{n} \int_{\mathbb{R}^d} \sum\limits_{i=1}^n \left\| \nabla_{\theta} U_{\X_i}(\theta) \right\|_2^{4} \rmd n_t(\theta) \right)^{\frac{1}{2}}\left( \frac{1}{n} \int_{\mathbb{R}^d} \sum\limits_{i=1}^n \1_{\|\nabla_{\theta} U_{\X_i}(\theta)\|_2 > g(t)} \rmd n_t(\theta) \right)^{\frac{1}{2}} \nonumber\\
&=& \left( \frac{1}{n}  \sum\limits_{i=1}^n \mathbb{E}\left[\left\| \nabla_{\theta} U_{\X_i}(\theta_t) \right\|_2^{4} \right] \right)^{\frac{1}{2}}  \left( \frac{1}{n} \sum_{i=1}^n  \mathbb{P} \left(\|\nabla_{\theta} U_{\X_i}(\theta_t)\|_2 > g(t) \right)  \right)^{\frac{1}{2}}.
\end{eqnarray}
\noindent 
We then use Proposition \ref{prop:growth_gradient_KL} and obtain that:
\begin{eqnarray*}
\tilde{\Delta}_t & \leq & \left(\frac{1}{n}\sum_{i=1}^n\mathbb{E}\left[\left(2 (nL+ \Ll) U_{\X_i}^{2}(\theta_t)\right)\right]\right)^{\frac{1}{2}}\left( \frac{1}{n} \sum\limits_{i=1}^n \mathbb{P} \left( 2(n L+\Ll) U_{\X_i}(\theta_t) > g^2(t) \right) \right)^{\frac{1}{2}} \\
& \leq & 2 (n L+\Ll) \left( n \mathbb{E} [U_{\nu_n}^{2}(\theta_t)] \right)^{\frac{1}{2}} \left( \frac{1}{n} \sum_{i=1}^n \frac{2(n L+\Ll)}{g^2(t)} \mathbb{E}  \left[ U_{\X_i} (\theta_t)\right]\right)^{\frac{1}{2}} \\
& \leq & [2(nL+\Ll)]^{\frac{3}{2}} n^\frac{1}{2} \frac{\mathbb{E}\left[U_{\nu_n}^{2}(\theta_t)\right]^{\frac{1}{2}}
\mathbb{E}\left[U_{\nu_n}(\theta_t)\right]^{\frac{1}{2}}}{g(t)},
\end{eqnarray*}
where we used the Markov's inequality and the relation $\|.\|_2 \leq \|.\|_1$ in $\mathbb{R}^n$.
We apply Proposition \ref{prop:moment_U} with $\alpha=2$ and $\alpha=1$ and obtain that a constant $C>0$ exists (whose value may change from line to line) such that: 
\begin{eqnarray*}
\tilde{\Delta}_t & \leq & 
C \frac{n^\frac{7}{2} \left(\lnb \right)^{\frac{3(1+r)}{2}}}{g(t)}.
\end{eqnarray*}
We use this last bound in \eqref{eq:Majo_Delta_t} and we deduce that:
$$
\Delta_t \leq g^2(t) I_t + C \frac{n^\frac{11}{2} \left( \lnb \right)^{\frac{3(1+r)}{2}} }{g(t)}.
$$
Optimizing this last bound with respect to $g(t)$ leads to the upper bound:
$$
\Delta_t \leq C I_t^{\frac{1}{3}} n^{\frac{11}{3}} \left( \lnb \right)^{1+r}, \quad \forall t \ge 0.$$
\end{proof}

\subsection{Evolution of the weighted $\mathbb{L}^2$ distance $I_t$}
The quantity $I_t$ involved in Proposition \ref{prop:diff_Jt} measures how close to $\nu_n$ the conditional distribution of $X_t \vert \theta_t$ is. To study $I_t$, we first remark that it may be rewritten in a simpler way.
\begin{eqnarray}
I_t &=& \int_{\mathbb{R}^d} \sum\limits_{i=1}^{n} \left( \frac{m_t(\X_i|\theta)}{\nu_n (\X_i)} - 1 \right)^2 \nu_n (\X_i) \ \rmd n_t(\theta) \nonumber\\
&=& \int_{\mathbb{R}^d} \sum\limits_{i=1}^{n} \left( \frac{m_t^2(\X_i|\theta)}{\nu_n^2 (\X_i)} -2\frac{m_t(\X_i|\theta)}{\nu_n (\X_i)} + 1 \right) \nu_n (\X_i) \ \rmd n_t(\theta) \nonumber\\
&=& \int_{\mathbb{R}^d} \sum\limits_{i=1}^{n} \left( \frac{m_t^2(\X_i|\theta)}{\nu_n (\X_i)} -2 m_t(\X_i|\theta) + \nu_n (\X_i) \right) \rmd n_t(\theta) \nonumber \\
&=& \int_{\mathbb{R}^d} \left( \sum\limits_{i=1}^{n}  \frac{m_t^2(\X_i|\theta)}{\nu_n (\X_i)} -1 \right) \rmd n_t(\theta) \nonumber \\
&=& \int_{\mathbb{R}^d} \sum\limits_{i=1}^{n}  \frac{m_t^2(\X_i|\theta)}{\nu_n (\X_i)} \rmd n_t(\theta) -1. \nonumber
\end{eqnarray}
Using that $m_t(\X_i|\theta) n_t(\theta) = m_t(\theta,\X_i)$ and $\nu_n(\X_i) = \frac{1}{n}$ for $i=1,2,\ldots,n$, we obtain that:
\begin{equation}\label{eq:It_autre_formule}
I_t = n \int_{\mathbb{R}^d} \sum\limits_{i=1}^{n}  \frac{m_t^2(\theta, \X_i)}{ n_t(\theta)} \rmd \theta -1.
\end{equation}
The next proposition then assesses how fast $I_t$ decreases to $0$ as $t \longrightarrow + \infty$.

\begin{proposition}\label{prop:diff_It}
For any $t \ge 0$:
\begin{equation}   I_t \leq I_0 e^{-2\alpha_n t} \leq (n-1) e^{-2\alpha_n t} .
\end{equation}
\end{proposition}

\begin{proof}
Our starting point is Equation \eqref{eq:It_autre_formule}. We compute its derivative  with respect to $t$:
\begin{eqnarray}
\partial_{t} \{  I_t\} &=& 2n \int_{\mathbb{R}^d} \sum\limits_{i=1}^{n} \frac{m_t(\theta, \X_i)}{n_t(\theta)} \partial_t m_t(\theta, \X_i) \rmd \theta - n \int_{\mathbb{R}^d} \sum\limits_{i=1}^{n}  \frac{m_t^2(\theta, \X_i)}{n_t^2(\theta)} \partial_t n_t(\theta) \rmd \theta \nonumber\\
&=& 2n \int_{\mathbb{R}^d} \sum\limits_{i=1}^{n} m_t(\X_i|\theta) \partial_t m_t(\theta, \X_i) \rmd \theta - n \int_{\mathbb{R}^d} \sum\limits_{i=1}^{n}  m_t^2(\X_i|\theta) \partial_t n_t(\theta) \rmd \theta. \nonumber
\end{eqnarray}
Using the Kolmogorov backward equation in the first line and $\mathcal{L} = \mathcal{L}_{1} + \mathcal{L}_{2}$ in the second one where $\mathcal{L}_1$ and $\mathcal{L}_2$ are defined in Equation \eqref{def:L}, we have:
\begin{eqnarray}
\partial_{t} \{  I_t\} &=& 2n \int_{\mathbb{R}^d} \sum\limits_{i=1}^{n} \mathcal{L} m_t(\X_i|\theta) \ m_t(\theta, \X_i) \rmd \theta - n \int_{\mathbb{R}^d} \sum\limits_{i=1}^{n}  m_t^2(\X_i|\theta) \partial_t n_t(\theta) \rmd \theta \nonumber\\
&=& \underbrace{2n \int_{\mathbb{R}^d} \sum\limits_{i=1}^{n} \mathcal{L}_{1} m_t(\X_i|\theta) \ m_t(\theta, \X_i) \rmd \theta}_{:=I_{3,t}} + \underbrace{2n \int_{\mathbb{R}^d} \sum\limits_{i=1}^{n} \mathcal{L}_{2} m_t(\X_i|\theta) \ m_t(\theta, \X_i) \rmd \theta}_{:=I_{1,t}} \nonumber \\
&& \underbrace{- n \int_{\mathbb{R}^d} \sum\limits_{i=1}^{n} m_t^2(\X_i|\theta) \partial_t n_t(\theta) \rmd \theta}_{:=I_{2,t}}. \label{eq:It_diff}
\end{eqnarray}
Then, $\partial_{t} \{ I_t\}$ may be splitted into three terms that are studied separately. 
\begin{itemize}
\item Study of $I_{1,t}$. We observe that: 
\begin{equation}\label{eq:L1_mt}
\mathcal{L}_2 m_t(\X_i|\theta) = \frac{\alpha_n}{n} \sum\limits_{j=1}^{n} [ m_t(\X_j|\theta) - m_t(\X_i|\theta)] = \frac{\alpha_n}{n} - \alpha_n \ m_t(\X_i|\theta).
\end{equation}
We then use this last equation in the definition of $I_1(t)$ and obtain that:
\begin{eqnarray}
I_{1,t} &=& 2n \int_{\mathbb{R}^d} \sum\limits_{i=1}^{n} \mathcal{L}_{2} m_t(\X_i|\theta) \ m_t(\theta, \X_i) \rmd \theta \nonumber \\
&=& 2\alpha_n \int_{\mathbb{R}^d} \sum\limits_{i=1}^{n} m_t(\theta, \X_i) \rmd \theta - 2\alpha_n n \int_{\mathbb{R}^d} \sum\limits_{i=1}^{n} m_t(\X_i|\theta) m_t(\theta, \X_i) \rmd \theta \nonumber\\
&=& 2\alpha_n - 2\alpha_n n \int_{\mathbb{R}^d} \sum\limits_{i=1}^{n}  \frac{m_t^2(\theta, \X_i)}{n_t(\theta)} \rmd \theta \nonumber\\
&=& -2\alpha_n I_t. \label{eq:I1_ineq}
\end{eqnarray}
		
\item Study of $I_{2,t}$. 
Using the definition of $n_t$, we obtain that:
\begin{eqnarray}
I_{2,t} &=& - n \int_{\mathbb{R}^d} \sum\limits_{i=1}^{n}  m_t^2(\X_i|\theta) \partial_t n_t(\theta) \rmd \theta \nonumber\\
&=& - n \int_{\mathbb{R}^d} \sum\limits_{i=1}^{n}  m_t^2(\X_i|\theta) \partial_t \left( \sum\limits_{j=1}^{n} m_t(\theta,\X_j) \right) \rmd \theta \nonumber\\
&=& - n \int_{\mathbb{R}^d} \sum\limits_{j=1}^{n} \sum\limits_{i=1}^{n}  m_t^2(\X_i|\theta)  \partial_t m_t(\theta,\X_j) \rmd \theta \nonumber\\
&=& - n \int_{\mathbb{R}^d} \sum\limits_{j=1}^{n} \left(\sum\limits_{i=1}^{n}  \mathcal{L} m_t^2(\X_i|\theta) \right) m_t(\theta,\X_j) \rmd \theta \nonumber\\
&=& - n \int_{\mathbb{R}^d} \sum\limits_{i=1}^{n}  \mathcal{L} m_t^2(\X_i|\theta) \ \rmd n_t(\theta). \nonumber
\end{eqnarray}
where we used the Kolmogorov backward equation in the fourth line and again the definition of $n_t$ in the last line. Again, the decomposition $\mathcal{L} = \mathcal{L}_1+\mathcal{L}_2$ yields: 
\begin{eqnarray*}
I_{2,t} & = & - n \int_{\mathbb{R}^d} \sum\limits_{i=1}^{n}  \mathcal{L}_{1} m_t^2(\X_i|\theta) \ \rmd n_t(\theta)
 - n \int_{\mathbb{R}^d} \sum\limits_{i=1}^{n}  \mathcal{L}_{2} m_t^2(\X_i|\theta) \ \rmd n_t(\theta) .\\
\end{eqnarray*}
We repeat some similar computations as those developed in Equation \eqref{eq:L1_mt} to study the action of the jump component induced by $\mathcal{L}_2$  on $m_t^2$. We obtain that:
$$\mathcal{L}_{2} m_t^2(\X_i|\theta) =   \frac{\alpha_n}{n}\sum\limits_{k=1}^{n} [ m_t^2(\X_k|\theta) - m_t^2(\X_i|\theta)] =
\frac{\alpha_n}{n}\sum\limits_{k=1}^{n} m_t^2(\X_k|\theta) - \alpha_n \ m_t^2(\X_i|\theta).$$
We use this last equation and obtain that:
\begin{eqnarray}
I_{2,t} & = & - n \int_{\mathbb{R}^d} \sum\limits_{i=1}^{n}  \mathcal{L}_{1} m_t^2(\X_i|\theta) \ \rmd n_t(\theta) - \alpha_n \int_{\mathbb{R}^d} \sum\limits_{i=1}^{n} \sum\limits_{k=1}^{n} m_t^2(\X_k|\theta) \ \rmd n_t(\theta) \nonumber \\
&& + \alpha_n n \int_{\mathbb{R}^d} \sum\limits_{i=1}^{n} m_t^2(\X_i|\theta)  \ \rmd n_t(\theta) \nonumber\\
& = & - n \int_{\mathbb{R}^d} \sum\limits_{i=1}^{n}  \mathcal{L}_{1} m_t^2(\X_i|\theta) \ \rmd n_t(\theta) - \alpha_n n \int_{\mathbb{R}^d} \sum\limits_{k=1}^{n} m_t^2(\X_k|\theta) \ \rmd n_t(\theta) \nonumber \\
&&+ \alpha_n n \int_{\mathbb{R}^d} \sum\limits_{i=1}^{n} m_t^2(\X_i|\theta) \ \rmd n_t(\theta) \nonumber\\
& =  & - n \int_{\mathbb{R}^d} \sum\limits_{i=1}^{n}  \mathcal{L}_{1} m_t^2(\X_i|\theta) \ \rmd n_t(\theta).
\end{eqnarray}		

\item Study of $I_{2,t} + I_{3,t}$. 
We observe that this sum involves only $\mathcal{L}_1$ (see Equation \eqref{def:L}. We first compute:
$$
\mathcal{L}_{1} m_t(\X_i|\theta) = -  \langle \nabla_{\theta} U_{\X_i}(\theta) , \nabla_{\theta} m_t(\X_i|\theta) \rangle + \Delta_{\theta} m_t(\X_i|\theta), 
$$	
and	similarly:
\begin{align*} \mathcal{L}_{1} m_t^2(\X_i|\theta)  & = - \langle \nabla_{\theta} U_{\X_i}(\theta) , \nabla_{\theta} m_t^2(\X_i|\theta), \rangle + \Delta_{\theta} m_t^2(\X_i|\theta) \\
& =  -2 m_t (\X_i|\theta) \langle\nabla_{\theta} U_{\X_i}(\theta) , \nabla_{\theta} m_t(\X_i|\theta)  \rangle + 2 \|\nabla_{\theta} m_t(\X_i|\theta)\|_2^2 + 2 m_t(\X_i|\theta) \Delta_{\theta} m_t(\X_i|\theta). \end{align*}
Using these two equations into $I_{2,t}+I_{3,t}$ and $m_t(\X_i \vert \theta) n_t(\theta)  = m_t(\theta,\X_i)$, we get:
\begin{align*}
\frac{I_{2,t}+I_{3,t}}{n} & = 2 \int_{\mathbb{R}^d} \sum\limits_{i=1}^{n}  \langle\nabla_{\theta} m_t(\X_i|\theta), \nabla_{\theta} U_{\X_i}(\theta)\rangle  m_t (\theta,\X_i) \rmd \theta \\ 
& - 2 \int_{\mathbb{R}^d} \sum\limits_{i=1}^{n} \|\nabla_{\theta} m_t(\X_i|\theta)\|_2^2 \ n_t(\theta) \rmd \theta - 2 \int_{\mathbb{R}^d} \sum\limits_{i=1}^{n} \Delta_{\theta} m_t(\X_i|\theta) \ m_t (\theta, \X_i) \rmd \theta  \\  
&  -2 \int_{\mathbb{R}^d} \sum\limits_{i=1}^{n} \langle\nabla_{\theta} m_t(\X_i|\theta), \nabla_{\theta} U_{\X_i}(\theta)\rangle m_t(\theta, \X_i) \rmd \theta + 2 \int_{\mathbb{R}^d} \sum\limits_{i=1}^{n} \Delta_{\theta} m_t(\X_i|\theta) \ m_t(\theta, \X_i) \rmd \theta \\
& = -  \int_{\mathbb{R}^d} \sum\limits_{i=1}^{n} \|\nabla_{\theta} m_t(\X_i|\theta)\|_2^2 \ \rmd n_t(\theta) \leq 0.
\end{align*}
\end{itemize}
Gathering this last inequality with \eqref{eq:I1_ineq} into Equation \eqref{eq:It_diff} yields:
\begin{equation}
\partial_t\{ I_t\} \leq -2\alpha_n I_t. \nonumber
\end{equation}
We conclude with a direct application of the Gronwall lemma while  observing that $I_0 \leq n-1$.
\end{proof}

\subsection{Functional (weak) log-Sobolev inequalities}

\subsubsection{Related works on functional inequalities}

A straightforward consequence of Proposition \ref{prop:diff_Jt} and Proposition \ref{prop:diff_It} is the following differential inequality on the relative entropy $J_t$:
\begin{equation}
\label{eq:diff_Jt+It}
\partial_{t} \{  J_t\} \leq - \int_{\mathbb{R}^d} \left\|\nabla_{\theta} \left( \sqrt{\frac{n_t(\theta)}{\mu_{n} (\theta)}} \right)\right\|_2^2 \rmd \mu_{n}(\theta) +  c_{n,d}  e^{-\frac{2\alpha_n}{3} t},
\end{equation}
where $c_{n,d}$ is defined as:
\begin{equation}\label{def:c_nd}
c_{n,d} \led n^{4} \left( \lnb \right)^{1+r}.
\end{equation}

At this stage, we should observe that a standard approach consists in finding a functional inequality that relates the key Dirichlet form $\mathcal{E}(f)$ defined by:
\begin{equation}
\label{def:dirichlet_form}
\mathcal{E}(f) = \int_{\mathbb{R}^d} \| \nabla_\theta f(\theta)\|_2^2 \rmd \mu_n(\theta), 
\end{equation}
to $Ent_{\mu_n}(f^2)$, the entropy itself with respect to $\mu_{n}$.
These approaches rely on the initial works of \cite{Gross} where Logarithmic Sobolev Inequality (LSI for short) were introduced. The consequences of LSI to exponential ergodicity has then been an extensive field of research and we refer to \cite{bakry2014analysis} for an overview on this topic. A popular sufficient condition that ensures LSI is the log strong-convexity of the targeted measure (see among other \cite{BakryEmery}) and an impressive amount of literature has been focused on the existing links  between these functional inequalities, ergodicity of the semi-group, transport inequalities and Lyapunov conditions. We refer to \cite{cattiaux2017hitting,bakry2008rate} (these two works are far from being exhaustive). The great interest of LSI has then been observed in machine learning and statistics  more recently as testified by the recent works in Monte Carlo samplings of \cite{ma2019sampling,mou2022improved}. A popular way to extend LSI from the strongly convex situation to a more general case relies on the ``strong convexity outside a ball''  hypothesis using the perturbation argument of the seminal contributions of  \cite{Holley_Stroock}. If this method proves to be suitable for the study of the simulated annealing process in \cite{miclo1992recuit}, \cite{Holley_Stroock}, it appears to be doubtful for the study of sampling problems with convex potentials that satisfies $\Hklsimple$ as this settings do not imply an asymptotic strong convexity of $\theta \longmapsto U(\theta)$ for large values of $\|\theta\|_2$. That being said, and maybe an even worst consequence of such approach, is the unavoidable dependency on the dimension for the LSI constant when using a perturbation approach, which leads to a serious exponential degradation of the convergence rates with the dimension of the ambient space.

To overcome these difficulties, we have chosen to use a slightly different functional inequality that may be considered as an innocent modification of LSI, but that indeed appears to be well suited to weakly log-concave setting described through an $\Hklsimple$ assumption. For this purpose, we shall use weak log-Sobolev inequalities (WLSI for short below) that have been introduced in \cite{WANG2000219}  and whose interest has been extensively studied in many works to obtain exponentially sub-linear rates of mixing, see among others for example \cite{cattiaux2007weak}.
To derive such inequalities, our starting point will be the contribution of \cite{CATTIAUX20091821} that makes the link between Lyapunov conditions and WLSI. Our approach based on $\Hklsimple$ certainly shares some similarities with the recent work of \cite{cattiaux2022self} where some functional inequalities (Poincaré and Transport inequalities) are obtained within a framework of variable curvature bound.

\subsubsection{Weak log Sobolev inequalities}

We briefly introduce the key theoretical ingredients, that are exhaustively described in \cite{bakry2014analysis}. We introduce the following assumption, that will be suitable for the setting of bounded functions.
\begin{definition}[Weak Log-Sobolev Inequality \label{def:WLSI}]
For any measurable space $(\Omega,\mathcal{F},\mu)$ and for any nice function $f$, let us define:
$$
Ent_\mu(f^2) := \int_{\Omega} f^2 \log(f^2) \rmd \mu
- \int_{\Omega} f^2 \rmd \mu \log \left( \int_{\Omega} f^2 \rmd \mu \right).$$
The measure $\mu$ satisfies a WLSI if a non-increasing function $\varphi_{WLS}:(0,+\infty) \mapsto \mathbb{R}_+$ exists such that for any $f \in \mathcal{C}^\1_b(\Omega)$:
\begin{equation}\label{def:WLSI}
Ent_\mu(f^2) \leq \varphi_{WLS}(s) \mathcal{E}(f)+ s \, Osc^2(f),\end{equation}
where $Osc(f): = \sup f - \inf f$.
\end{definition}

Before establishing how to use this functional inequality, we first state the important relationship between Poincar\'e Inequality and WLSI.
 
\begin{proposition}\label{prop:PI_WLSI}
Assume that $\mu$ satisfies a Poincar\'e Inequality of constant $C_P$, \ie for any smooth integrable function $f$:
$$
C_p(\mu) Var_{\mu}(f) = C_p(\mu) \int_{\Omega} (f-\mu[f])^2 \rmd \mu \leq \int_{\Omega} |\nabla f|^2 \rmd \mu,
$$
then if $\log c = \frac{3}{14e^2}\left( \frac{1}{e}+\frac{1}{2} \right)+1+\log \left(\frac{14}{3} \right)$, then $\mu$ satisfies a WLSI with:
\begin{equation*}
\varphi_{WLS}(s) = \left\{ \begin{array}{r l}
0, & s > \frac{1}{e}+\frac{1}{2}\\
\frac{32}{C_P} \log\left(\frac{c}{s}\right)  , & s \leq \frac{1}{e}+\frac{1}{2} \\
\end{array} \right. .
\end{equation*}  
For the sake of readability, we introduce a universal $\a>0$ such that:
\begin{equation}
\label{eq:PI_WLSI}
\varphi_{WLS}(s)  = \left\{ \begin{array}{r l}
0, & s > \frac{1}{e}+\frac{1}{2}\\
\a \frac{1+ \log\left(\frac{1}{s}\right)}{C_P}, & s \leq \frac{1}{e}+\frac{1}{2} \\
\end{array} \right. .
\end{equation}
\end{proposition}

\begin{proof}[Proof of Proposition \ref{prop:PI_WLSI}]
The proof of how the Poincar\'e Inequality implies the WLSI in the bounded setting described in Definition \ref{def:WLSI} is given for the sake of completeness. Technical details are skipped and we refer to the references below. 
We use the measure-capacity inequality (see \cite{bakry2014analysis}, Section 8.3). We know that the Poincar\'e Inequality implies a capacity inequality (Proposition 8.3.1 of \cite{bakry2014analysis}) with a constant equal to $2 C_P$. Then, we can apply Theorem 2.2 of \cite{cattiaux2007weak} that induces a WLSI which is based on the function $\varphi_{WLS}$ given in the statement of the proposition.
\end{proof}

\subsubsection{Weak log Sobolev inequalities under $\Hklsimple$}
Of course, in the previous result, the only important dependency will be the one induced by $C_P$, which will deserve an ad-hoc study under Assumption $\Hklsimple$. The numbers 32 and $\log(c)$ will be dealt with as ``universal constants'' in what follows.  

The next proposition states two lower bounds on the Poincar\'e constant within the $\Hklsimple$ framework.
The first one always holds, regardless the value of $(X_1,\ldots,X_n)$ that may be been randomly sampled. The second one has to be considered with high probability, with respect to the sampling process $(X_1,\ldots,X_n)$.

\begin{proposition} \label{prop:constante_poincare_borne} Assume $\Hmin$,$\Hn$, $\Hpi$ and for any $x$, $\theta \mapsto -\log p_{\theta}(x)$ satisfies $\Hklsimple$, then: 
\begin{itemize}
\item[$i)$] For any sample $(X_1,\ldots,X_n)$, it holds:
$$
C_P(\mu_n) \ged \frac{1}{ \left( \lnb \right)^{(1+r)^2} }
$$
\item[$ii)$] Assume that $\theta \mapsto \mathbb{P}_{\theta}$ is injective and $\theta_0$ exists such that $(X_1,\ldots,X_n) \sim \mathbb{P}_{\theta_0}$. If locally around $\theta_0$, 
$\theta \mapsto |\theta-\theta_0|^{- \alpha} W_1(\mathbb{P}_\theta,\mathbb{P}_{\theta_0})$ does not vanish, then:
$$
\mathbb{E}_{(X_1,\ldots,X_n) \sim \mathbb{P}_{\theta_0}}[C_P(\mu_n)] \ged \left(\frac{n}{L d \log n}\right)^{\alpha}.
$$
\end{itemize}   
\end{proposition}

We are finally led to upper bound the oscillations of the function involved in the WLSI introduced in \eqref{def:WLSI}, \ie we are looking for an upper bound of $Osc^2\left(\sqrt{\frac{n_t}{\mu_n}}\right)$ for any time $t>0$. For this purpose, we observe that the Markov semi-group induces that $f_t = \frac{n_t}{\mu_n} = P_t f_0$ where $f_0=\frac{n_0}{\mu_n}$.
The next proposition implies the boundedness of $f_t$ over $\mathbb{R}^d$ when $n_0$ is chosen as a Gaussian distribution with a carefully tuned covariance matrix.

\begin{proposition}\label{prop:oscillation_borne} Assume $\Hmin$,$\Hn$, $\Hpi$ and that, for any $x$, $\theta \mapsto -\log p_{\theta}(x)$ satisfies $\Hklsimple$, then:
\begin{itemize}
\item[$i)$] Two positive constants $C_1$ and $C_2$ exist, which are independent from $n$ and $d$ and such that: 
$$
\|f_0\|_{\infty} \led \left(\frac{C_1 d }{ n}\right)^{\frac{d r }{2}} \exp\left(C_2 n d^{1+r} \log^{2\beta(1+r)} (n) \right).
$$

\item[$ii)$] As a consequence:
$$
Osc^2(\sqrt{f_t}) \leq Osc^2(\sqrt{f_0}) \led \left(\frac{C_1 d }{ n}\right)^{\frac{d r }{2}} \exp\left(C_2 n d^{1+r} \log^{2\beta(1+r)} (n) \right).
$$
\item[$iii)$] Moreover, a straightforward consequence of $i)$ is: 
$$ 
J_0 = \int_{\mathbb{R}^d} \log \left( f_0(\theta) \right) \rmd n_0 (\theta) \led n d^{1+r} \log^{2\beta(1+r)} (n) + d \log \left(\frac{d }{ n}\right).
$$
\end{itemize}
\end{proposition}

\subsection{Entropic convergence of the SLMC}
The purpose of this paragraph is to prove the main result of the paper, \ie Theorem \ref{theo:convergence} that guarantees the convergence of the SLMC algorithm. 

\begin{proof}[Proof of Theorem \ref{theo:convergence}]
Our starting point is the semi-group inequality \eqref{eq:diff_Jt+It} associated with the functional WLSI inequality \eqref{def:WLSI}. Using $c_{n,d}$ defined in \eqref{def:c_nd}, we obtain for any $s>0$:
\begin{align*}
\partial_t\{J_t\} &\leq - \mathcal{E}\left( \sqrt{\frac{n_t}{\mu_n}} \right) + c_{n,d} e^{-\frac{2 \alpha_n}{3} t} \\
& \leq - \frac{J_t}{\varphi_{WLS}(s)} + \frac{s}{\varphi_{WLS}(s)} Osc^2\left(\sqrt{\frac{n_t}{\mu_n}}\right) +  c_{n,d} e^{-\frac{2 \alpha_n}{3} t} \\
& \leq - \frac{J_t}{\varphi_{WLS}(s)} + \frac{s\, O_{n,d}}{\varphi_{WLS}(s)}+  c_{n,d} e^{-\frac{2 \alpha_n}{3} t},
\end{align*}
where we applied Proposition \ref{prop:oscillation_borne} in the last line with $O_{n,d} \led \left(\frac{C_1 d }{ n}\right)^{\frac{d r }{2}} \exp\left(C_2 n d^{1+r} \log^{2\beta(1+r)} (n) \right)$ and $C_1$ and $C_2$ two universal constants. We then choose $s$ (that depends on $t$) such that:
$$
s_t=e^{-A \sqrt{t+1}} \quad \text{with} \quad A>1 \quad \text{that will be chosen later on}.
$$
We observe that $s_t < e^{-1}+1/2$, so that Equation \eqref{eq:PI_WLSI} of Proposition \ref{prop:PI_WLSI}
yields:
$$
\varphi_{WLS}(s_t) = \a \frac{1+ \log\left( \frac{1}{s_t}\right)}{C_P} = \a \frac{1+ A \sqrt{1+t}}{C_P}.
$$
We introduce $\psi(t)=\exp \left( \frac{C_P}{\a} \int_{0}^t  \frac{\rmd u}{1+ A \sqrt{1+u}} \right)$ and deduce that
$$
\psi(t) = 
\exp \left(\frac{C_P}{\a} \frac{2A(\sqrt{1+t}-1) - 2 \log\left(\frac{1+A\sqrt{1+t}}{1+A}\right)}{A^2} \right)  \leq \exp \left(\frac{ 2 C_P}{\a A}(\sqrt{1+t}-1)\right).
$$
We now apply the Gronwall Lemma:
\begin{align*}
\partial_t \left\{ \psi(t) J_t \right\}& = \left( \frac{C_P}{\a (1+ A \sqrt{1+t})} J_t + J'_t \right) \psi(t) \\
& \leq \left[ \frac{C_P  O_{n,d} }{\a} \frac{ e^{-A \sqrt{t+1}}}{1+ A \sqrt{1+t}} +  c_{n,d} e^{-\frac{2 \alpha_n}{3} t} \right]\psi(t) \\
& \leq  \frac{C_P  O_{n,d} }{\a}  e^{-(A-\frac{ 2 C_P}{\a A}) \sqrt{1+t} } +  c_{n,d} e^{\frac{ 2 C_P}{\a A} (\sqrt{1+t}-1)-\frac{2 \alpha_n}{3} t}.
\end{align*}
We denote by $t_0$ the positive real value that solves the equation $\frac{2 C_P}{\a A} \sqrt{1+t_0}  = \frac{\alpha_n t_0}{3}.$  
We then observe that:
\begin{align*}
\int_{0}^t e^{\frac{ 2 C_P}{\a A} (\sqrt{1+u}-1)-\frac{2\alpha_n}{3} u} \rmd u &\leq 
\int_{0}^{t_0} e^{\frac{ 2 C_P}{\a A} \sqrt{1+u}} \rmd u + \int_{t_0}^{+ \infty} e^{-\frac{\alpha_n}{3} u} \rmd u \\
& \leq t_0 e^{\frac{ 2 C_P}{\a A} \sqrt{1+t_0}} + \frac{3}{\alpha_n} = t_0 e^{\frac{\alpha_n t_0}{3}} + \frac{3}{\alpha_n}.
\end{align*}
If $A$ is chosen such that $A>\frac{ 2 C_P}{\a A} $, we then deduce that:
\begin{align*}
J_t & \leq \left( J_0 +c_{n,d} t_0 e^{\frac{\alpha_n t_0}{3}} + \frac{3 c_{n,d}}{\alpha_n}\right) \psi(t)^{-1} + \frac{C_P O_{n,d} }{\a} \psi(t)^{-1} \int_{0}^t e^{-\left(A-\frac{ 2 C_P}{\a A}\right) \sqrt{1+u} } \rmd u \\
& \leq \left( J_0 + c_{n,d} t_0 e^{\frac{\alpha_n t_0}{3}} + \frac{3 c_{n,d}}{\alpha_n}\right) \psi(t)^{-1}
+ \frac{2 C_P  O_{n,d} }{\a \left(A-\frac{ 2 C_P}{\a A}\right)^2} \psi(t)^{-1},
\end{align*}
where we used in the previous line the bound:
$$
\int_{0}^t  e^{- b  \sqrt{1+u} } \rmd u \leq \int_{0}^{ + \infty}  e^{- b  \sqrt{1+u} } \rmd u \leq \frac{2}{b^2}.
$$

To obtain the lowest upper bound, we are led to choose $A$ such that $\frac{ 2 C_P}{\a A}$ as large as possible and below $A$, which naturally drives to the choice:
$$
\frac{ 2 C_P}{\a A} = \frac{A}{2} \Longrightarrow A = \frac{2}{\sqrt{\a}} \sqrt{C_P}.
$$
Using this value of $A$ in the previous bound, we observe that $t_0 \leq \frac{3 \sqrt{C_P}}{\alpha_n \sqrt \a} + \frac{C_P}{\alpha_n^2}$, so that a constant $C$ exists such that:
\begin{equation}\label{eq:borne_Jt}
J_t \leq C \left( J_0 + \frac{c_{n,d}}{\alpha_n } \left[1+\left(\frac{C_P}{\alpha_n}+\sqrt{C_P}\right) e^{\frac{\sqrt{C_P}}{\sqrt \a} + \frac{C_P}{3\alpha_n}}\right] + O_{n,d} \right) (1+t)^{1/4} e^{-\frac{\sqrt{C_p}}{\sqrt{a}} (\sqrt{1+t}-1)}.
\end{equation}

In Proposition \ref{prop:constante_poincare_borne} we obtained $C_P \geq \frac{\kappa}{\left( \lnb \right)^{(1+r)^2}}$. If instead of using the constant $C_P$, we use directly $\frac{\kappa}{\left( \lnb \right)^{(1+r)^2}}$ with $\kappa<1$, then all the previous computations remain the same only replacing $C_P$ by its lower bound and:
\begin{equation}
J_t \leq C \left( J_0 + \frac{c_{n,d}}{\alpha_n }e^{\frac{ \sqrt{\kappa} \left( \frac{1}{\sqrt{a}} + \frac{1}{3\alpha_n} \right)}{ \left( \lnb \right)^{(1+r)^2/2}} } + O_{n,d} \right) (1+t)^{1/4} e^{-\frac{\sqrt{\kappa}(\sqrt{1+t}-1)}{\sqrt{a} \left( \lnb \right)^{(1+r)^2/2}} }.
\end{equation}

Using the values of $O_{n,d}$, $c_{n,d}$ and the upper bound of $J_0$, we finally observe that if $\alpha_n = \frac{1}{n\left( \lnb \right)^{1+r}}$, then:
$$
t \ge \ct \left( \lnb \right)^{(1+r)^2} \left[\log^2(\varepsilon^{-1}) + n^2 \left( \lnb \right)^{2(1+r)} + d^2 \log^2 d \right]  \Longrightarrow J_t \leq \varepsilon.
$$
\end{proof}

\section{Technical results on KL and $U_{\nu_n}$ \label{section:results on KL and U}}

\subsection{Growth properties under the Kurdyka-\L ojasiewicz inequality}

We remind here some important consequences of the KL inequality that implies several relationships between the function and the norm of its gradient. The proof of these inequalities may be found in Lemma 15 of \cite{GPP} (a small mistake appears and we correct the statement with a factor $2$ in our work).

\begin{proposition}\label{prop:growth_gradient_KL}
Assume that a function $V$ satisfies $\Hklsimple$, then:
$$
\frac{2 \c}{1-r} \left[ V^{1-r}(\theta) - \min(V)^{1-r}\right] \leq \|\nabla V(\theta)\|^2_2 \leq 2 L \left[ V(\theta)-\min(V)\right], \quad \forall \theta \in \mathbb{R}^d.
$$
\end{proposition}
\noindent 
It is furthermore possible to assess a minimal and maximal growth property of any function that satisfies $\Hklsimple$, which is necessarily lower and upper bounded by a positive power of the distance to its minimizer.

\begin{proposition}\label{prop:growth_function_KL}
Assume that a function $V$ satisfies $\Hklsimple$, then, $\forall \theta \in \mathbb{R}^d$:
$$
V^{1+r}(\theta)-\min(V)^{1+r} \ge \frac{(1+r) \c}{2} \|\theta-\argmin V\|_2^{2},
$$
and
$$
V(\theta)-\min(V) \le \frac{L}{2} \|\theta-\argmin V\|_2^{2}.
$$
\end{proposition}
A straightforward consequence of the first inequality is then
\begin{proposition}\label{prop:growth_function_kl2}
Assume that a function $V$ satisfies $\Hklsimple$, then, $\forall \theta \in \mathbb{R}^d$:
$$
V(\theta) \geq 2^{-\frac{r}{1+r}} \left(\min(V)+\left(\frac{(1+r)\c}{2} \right)^{\frac{1}{1+r}} \|\theta-\argmin V\|_2^{\frac{2}{1+r}}\right).
$$
\end{proposition}

\subsection{Properties of $U_{\nu_n}$}

\begin{proof}[Proof of Proposition \ref{prop:kl_to_kl}]   
First, we observe that if each $\theta \mapsto \nabla \log p_\theta(\X_i)$ is $L$-Lipschitz and $\theta \mapsto \nabla \log \pi_0$ is $\Ll$-Lipschitz, then the triangle inequality implies that
$$
\|\nabla U_{\nu_n}(\theta_1)-\nabla U_{\nu_n}(\theta_2)\|_2 \leq (n L + \Ll) \|\theta_1-\theta_2\|_2.
$$
Second, we consider the lower-bound property on the curvature and observe that:
$$
\underline{\lambda}_{\nabla^2 U_{\nu_n}(\theta)}  = \inf_{e \in \mathbb{R}^d : |e|=1} e^{T} (\nabla^2 U_{\nu_n})(\theta) e \ge 
\frac{1}{n} \sum_{i=1}^n \inf_{e \in \mathbb{R}^d : |e|=1} e^{T} (\nabla^2 U_{\X_i})(\theta) e.
$$
The log concavity of the prior yields
$$
\underline{\lambda}_{\nabla^2 U_{\nu_n}(\theta)}  \ge \frac{1}{n} \sum_{i=1}^n \underline{\lambda}_{\nabla^2(- n \log p_{\theta}(\X_i))} = \sum_{i=1}^n \underline{\lambda}_{\nabla^2(-\log p_{\theta}(\X_i))}.
$$
Then, the $\Hklsimple$ property applied to each term of the sum above and $\min_{\theta \in \mathbb{R}^d} -\log \pi_0 (\theta) > 0$ yields 
$$
\underline{\lambda}_{\nabla^2 U_{\nu_n}(\theta)} \ge \c \sum_{i=1}^n [-\log p_{\theta}(\X_i)]^{-r} \geq \c n^r \sum_{i=1}^n U_{\X_i}^{-r} (\theta) = \c n^{1+r} \left( \frac{1}{n} \sum_{i=1}^n U_{\X_i}^{-r}(\theta) \right).
$$
From the Jensen inequality, we finally deduce that: 
$$
\underline{\lambda}_{\nabla^2 U_{\nu_n}(\theta)} \ge \c n^{1+r} \left( \frac{1}{n} \sum_{i=1}^n U_{\X_i}^{-r}(\theta) \right) \ge \c n^{1+r} U_{\nu_n}^{-r}(\theta).
$$
We conclude that $U_{\nu_n}$ satisfies $\Hkl$. For $U_{\X_i}$, the proof is similar. 
\end{proof}

\begin{proposition}\label{prop:growth_min_KL} We assume $\Hpi$, $\Hmin$ and that for any $x$: $\theta \longmapsto - \log p_\theta(x)$ satisfies $\Hklsimple$, then: 
$$
\|\argmin U_{\nu_n}\|_2 \led  d^{\frac{1+r}{2}} \log^{\beta(1+r)} (n) \quad \mbox{ and } \quad \min_{\theta \in \mathbb{R}^d} \ U_{\nu_n} (\theta) \led \MndU. 
$$ 
\end{proposition}

\begin{proof}
Proposition \ref{prop:kl_to_kl} shows that $U_{\nu_n}$ satisfies $\Hkl$. Therefore, we can apply Proposition \ref{prop:growth_function_KL} with $\theta=0$ and deduce that:
$$
\|\argmin U_{\nu_n} \|_2^2 \leq \frac{2}{(1+r)\c n^{1+r}} \left( U_{\nu_n}^{1+r}(0) - \min \ U_{\nu_n}^{1+r} \right).
$$
To obtain an upper bound of $U_{\nu_n} (0)$ we first bound $U_{\X_i}(0)$ using Proposition \ref{prop:growth_function_KL}, for all $i$, as follows:
$$
U_{\X_i} (0) \leq \min \ U_{\X_i} + \frac{n L + \Ll}{2} \|\argmin U_{\X_i}\|_2^2 \led \Mnd + n \lnb \led \MndU, 
$$
then $U_{\nu_n} (0) \led \MndU$. We deduce that:
$$
\|\argmin U_{\nu_n} \|_2^2 \leq \frac{2}{(1+r)\c n^{1+r}} U_{\nu_n}^{1+r}(0) \led d^{1+r} \log^{2\beta(1+r)} (n).
$$
The second part comes from $\min \ U_{\nu_n} \leq U_{\nu_n}(0)$.
\end{proof}

\section{Smoothness and boundedness of the semi-group\label{section:result semigroup}}

\begin{proof}[Proof of Proposition \ref{prop:constante_poincare_borne}]
\underline{$i)$.} The proof relies on an argument set up with a "fixed" sample $(X_1,\ldots,X_n)$.
Our starting point is Proposition \ref{prop:growth_function_KL} and the consequences of the Kurdyka-\L ojasiewicz inequality.
Since $\Hpi$ and $\theta \mapsto - \log p_\theta(\X_i)$ satisfies $\Hklsimple$, then Proposition \ref{prop:kl_to_kl} shows that $U_{\nu_n}$ satisfies $\Hkl$. Therefore, we can apply Proposition \ref{prop:growth_function_KL} and deduce that:
$$
\|\theta-\argmin U_{\nu_n} \|_2^2 \leq \frac{2}{(1+r) \c n^{1+r}} \left( U_{\nu_n}^{1+r}(\theta) - \min \, U_{\nu_n}^{1+r} \right)\leq \frac{2}{(1+r) \c n^{1+r}}
U_{\nu_n}^{1+r}(\theta). 
$$
If $I_d$ refers to the identity map, we use the fact that for any distribution $\mu$, we have $Var[\mu] \leq \mu[\|I_d-a\|_2^2]$ for any $a \in \mathbb{R}^d$ so that a straightforward consequence with $a=\argmin U_{\nu_n}$ is  then:
$$
Var(\mu_n) \leq \int_{\mathbb{R}^d} \|\theta-\argmin U_{\nu_n}\|_2^2  \rmd \mu_n (\theta) \leq \frac{2}{(1+r) \c n^{1+r}} 
\mu_n[U_{\nu_n}^{1+r}].
$$
We then use the ergodic behaviour of $(\theta_t)_{t \ge 0}$ and observe that there exists a constant $C$ independent from $n$ and $d$ such that: 
\begin{align*}
Var(\mu_n) &\leq \frac{2}{(1+r) \c n^{1+r}} \lim\sup_{t \ge 0} \mathbb{E}[U_{\nu_n}^{1+r}(\theta_t)]     \\
& \leq C  \left( \lnb \right)^{(1+r)^2},
\end{align*}
where the last inequality comes from Proposition \ref{prop:moment_U}.
We now use the Bobkov bound on the Poincaré constant for log-concave distribution (see Theorem 1.2 of \cite{Bobkov_AOP_99}) and deduce that
a universal constant $K$ exists such that:
$$
C_P(\mu_n) \ge \frac{1}{4 K^2 Var(\mu_n)}.
$$
Using the upper bound of the variance, we deduce that a universal $\kappa>0$ exists such that:
$$
C_P(\mu_n) \ge \frac{\kappa}{ \left( \lnb \right)^{(1+r)^2} }. 
$$

\noindent
\underline{$ii)$.} For the second point, we consider a situation on average over the samples and the result uses the concentration of the posterior distribution around its mean. We know from Theorem 3 of \cite{GPP} that
a constant $c>0$ exists such that:
$$
\mathbb{E}_{(X_1,\ldots,X_n) \sim \mathbb{P}_{\theta_0}}[ \mathbb{V}ar(\mu_n)] \leq c \epsilon_{n,d}^2,
$$
with $\epsilon_{n,d} = \left(\frac{L d \log n}{n}\right)^{\alpha^{-1}}$.
The result   follows using the Jensen inequality and the Bobkov bound.
\end{proof}

\begin{proof}[Proof of Proposition \ref{prop:oscillation_borne}]
\underline{$i)$.} We first establish the boundedness of $f_0$. From our assumptions, we apply Proposition \ref{prop:kl_to_kl} and obtain that $U_{\nu_n}$ satisfies $\Hkl$. If $\theta_n^\star = \argmin U_{\nu_n}$, we then deduce from Proposition \ref{prop:growth_function_KL} that:
\begin{align}
f_0(\theta) = \frac{n_0(\theta)}{\mu_n(\theta)} 
 = \frac{Z_n e^{-\frac{\|\theta\|_2^2}{2\sigma^2 } + U_{\nu_n}(\theta)}}{(2\pi)^{d/2} \sigma^d}  &\leq  \frac{Z_n e^{-\frac{\|\theta\|_2^2}{2\sigma^2 }+  U_{\nu_n}(\theta_n^\star)
+ \frac{(n L + \Ll)}{2} \|\theta-\theta_n^{\star}\|_2^2}}{(2\pi)^{d/2} \sigma^d}. \label{eq:tec_1_f0}
\end{align}
We compute an upper bound of $Z_n$ and use the lower bound of $U_{\nu_n}$ induced by Proposition \ref{prop:growth_function_kl2}:
\begin{align*}
Z_n & = \int_{\mathbb{R}^d} e^{- U_{\nu_n}(\theta)} \rmd \theta \\
& \leq  \int_{\mathbb{R}^d} e^{-2^{-\frac{r}{1+r}}\left[U_{\nu_n}(\theta_n^\star)+ n \left(\frac{(1+r) \c }{2}\right)^{\frac{1}{1+r}} \|\theta-\theta_n^\star\|_2^{\frac{2}{1+r}}\right]} \rmd \theta  \\
& \leq e^{-2^{-\frac{r}{1+r}} U_{\nu_n}(\theta_n^\star)} \int_{\mathbb{R}^d} e^{-n a_r \|\theta\|_2^{\frac{2}{1+r}}} \rmd \theta,
\end{align*}
with $a_r=\frac{\left((1+r) \c \right)^{\frac{1}{1+r}}}{2} $. Using the well known equality:
$$
\int_{\mathbb{R}^d} e^{-a |\theta|^\ell} \rmd \theta = \frac{d \pi^{d/2} \Gamma(d/\ell)}{\ell a^{d/\ell} \Gamma(d/2+1)}, \quad \forall a >0, \quad \forall \ell >0.
$$
we then deduce with $a=n a_r$ and $\ell=\frac{2}{1+r}$ that:
\begin{equation*}
Z_n \leq 
e^{-2^{-\frac{r}{1+r}} U_{\nu_n}(\theta_n^\star)} \int_{\mathbb{R}^d} e^{-n a_r \|\theta\|_2^{\frac{2}{1+r}}} \rmd \theta \leq \frac{d(1+r)}{2} \frac{\pi^{d/2}}{(n a_r)^{\frac{d(1+r)}{2}}} \frac{\Gamma\left(\frac{d(1+r)}{2}\right)}{\Gamma\left(\frac{d}{2}+1\right)}.
\end{equation*}
From standard relationships on the Gamma function:
\begin{equation}
Z_n \leq 2 \left( \frac{2^{1+r} \pi}{\c n^{1+r}} \right)^{\frac{d}{2}} d^{\frac{d r}{2}}.
\label{eq:tec_2_f0}
\end{equation}
We gather Equations \eqref{eq:tec_1_f0} and \eqref{eq:tec_2_f0} and obtain that: 
\begin{align*}
f_0(\theta) \leq 2e^{U_{\nu_n}(\theta_n^\star)} \left( \frac{2}{\c \sigma^2 n^{1+r}} \right)^{\frac{d}{2}} d^{\frac{d r}{2}} e^{-\frac{\|\theta\|_2^2}{2\sigma^2 }+ \frac{(n L + \Ll)}{2} \|\theta-\theta_n^{\star}\|_2^2}.
\end{align*}

For all $\sigma^2<\frac{1}{n L+ \Ll}$, a straightforward optimization on $\theta$ yields :
$$
\|f_0\|_{\infty} \leq 2e^{ U_{\nu_n}(\theta_n^\star)} \left( \frac{2}{\c \sigma^2 n^{1+r}} \right)^{\frac{d}{2}} d^{\frac{d r}{2}} \exp\left( \frac{(nL+\Ll)}{2(1-\sigma^2 (nL+\Ll)) } \|\theta_n^\star\|_2^2 \right).
$$
Then, the choice $\frac{c_1}{nL+\Ll} \leq \sigma^{2} \leq  \frac{c_2}{nL+\Ll}$, where $0<c_1\leq c_2<1$ in $\Hn$ and the bounds of $\|\theta_n^\star\|_2^2$ and $U_{\nu_n}(\theta_n^\star)$ in Proposition \ref{prop:growth_min_KL} lead to :
$$
\|f_0\|_{\infty} \leq 2 \left(\frac{C_1 d }{ n}\right)^{\frac{d r }{2}} \exp\left(C_2 n d^{1+r} \log^{2\beta(1+r)} (n) \right),
$$
where $C_1$ and $C_2$ are universal constants.

\noindent \underline{$ii)$.} This result is an almost standard consequence of the maximum principle for a Markov semi-group property with a Brownian diffusion. For any bounded measurable $h>0$, 
we observe that $P_t h>0$ using the Markov property, and we are led to define $g_t$ as the following function $g_t := \sqrt{P_t h}$. We then introduce $\overline{\theta}(t)$ and $\underline{\theta}(t)$ as: 
$$
\overline{\theta}(t) = \argmax g_t (\theta)
\quad \mbox{ and } \quad 
\underline{\theta}(t) = \argmin g_t (\theta).
$$
The chain rule yields:
\begin{eqnarray}
\frac{d}{d t} Osc(g_t) &=& \frac{d}{d t} \left( g_t(\overline{\theta}(t)) - g_t(\underline{\theta}(t)) \right) \nonumber \\
&=& \frac{d g_t}{d t} (\overline{\theta}(t)) + \left\langle \nabla g_t(\overline{\theta}(t)),  \frac{d \overline{\theta}(t)}{d t}\right\rangle  - \frac{d g_t}{d t} (\underline{\theta}(t)) - \left\langle \nabla g_t(\underline{\theta}(t)) , \frac{d \underline{\theta}(t)}{d t}\right\rangle .
\end{eqnarray}
We compute:
\begin{eqnarray}
\frac{d g_t}{d t} (\theta) &=& \frac{1}{2 \sqrt{P_t h}} \frac{d P_t h}{d t} (\theta) \nonumber \\
&=& \frac{1}{2 \sqrt{P_t h}} \Gt P_t h (\theta) \nonumber \\
&=& \frac{1}{2 \sqrt{P_t h (\theta)}} \left[ - \sum\limits_{i=1}^n \langle\nabla_{\theta} P_t h (\theta), \nabla_{\theta} U_{\X_i}(\theta)\rangle m_t(\X_i|\theta) + \Delta_{\theta} P_t h (\theta) \right].
 \end{eqnarray}
Now, we use that $\overline{\theta}(t) = \argmax g_t = \argmax P_t h$, (a similar argument holds for $\underline{\theta}(t)$):
$$
\nabla_{\theta} g_t(\overline{\theta}(t)) = 0, \quad 
\nabla_{\theta} P_t h (\overline{\theta}(t)) = 0 \quad \text{and} \quad \Delta_{\theta} P_t h (\overline{\theta}(t)) \leq 0.$$
then:
\begin{eqnarray}
\frac{d}{d t} Osc(g_t) &=& \frac{d g_t}{d t} (\overline{\theta}(t)) - \frac{d g_t}{d t} (\underline{\theta}(t)) \nonumber \\
&=& \frac{\Delta_{\theta} P_t h}{2 \sqrt{P_t h}} (\overline{\theta}(t)) - \frac{\Delta_{\theta} P_t h}{2 \sqrt{P_t h }} (\underline{\theta}(t)) \\
&\leq& 0. \nonumber
\end{eqnarray}
We have therefore shown that $Osc (\sqrt{P_t h})$ is decreasing in $t\geq 0$, which ends the proof.
\end{proof}

\begin{proof}[Proof of Lemma \ref{lemma:derivation_semi_group}]
We proceed as in Proposition 3 of \cite{miclo1992recuit} to justify the use of the Lebesgue dominated convergence theorem for the derivation of the integral involved in our statement.
We can then deduce that:
\begin{equation}
\partial_t \left\{ \int_{\mathbb{R}^d} f_t(\theta) \rmd n_t(\theta) \right\} = \int_{\mathbb{R}^d} \partial_t \{f_t(\theta)\} \rmd n_t(\theta) + \int_{\mathbb{R}^d} f_t(\theta) \partial_t \{n_t(\theta)\} \rmd \theta. \nonumber
\end{equation}
We leave the first term unchanged and now focus on the second term:
\begin{eqnarray}
\int_{\mathbb{R}^d} f_t(\theta) \partial_t \{n_t(\theta)\} \rmd \theta &=& \int_{\mathbb{R}^d} f_t (\theta) \partial_t \left\{ \sum\limits_{i=1}^n  m_t(\theta,\X_i) \right\} \rmd \theta \nonumber \\
&=& \int_{\mathbb{R}^d} \sum\limits_{i=1}^n  f_t(\theta) \partial_t \{m_t(\theta,\X_i)\} \rmd \theta \nonumber \\
&=& \int_{\mathbb{R}^d} \sum\limits_{i=1}^n \mathcal{L} f_t(\theta) \ m_t(\theta,\X_i) \rmd \theta, \nonumber
\end{eqnarray}
where we used the definition of $n_t$ in the first step and Kolmogorov backward equation \eqref{eq:kolmogorov} in the last one.	
Since the function $f_t(\theta)$ does not depend on $x$, we observe that $\mathcal{L}_{2} f_t(\theta) = 0$ and we only need to compute the remaining term $\mathcal{L}_{1} f_t(\theta)$:
\begin{eqnarray}
\int_{\mathbb{R}^d} f_t(\theta) \partial_t \{n_t(\theta)\} \rmd \theta &=& \int_{\mathbb{R}^d} \sum\limits_{i=1}^n \mathcal{L}_{1} f_t (\theta) \  m_t(\theta,\X_i) \rmd \theta \\
&=& \int_{\mathbb{R}^d} \sum\limits_{i=1}^n  \left[ - \langle\nabla_{\theta} f_t(\theta), \nabla_{\theta} U_{\X_i}(\theta)\rangle  + \Delta_{\theta} f_t(\theta) \right] m_t(\theta,\X_i) \rmd \theta \nonumber \\
&=& -  \int_{\mathbb{R}^d} \sum\limits_{i=1}^n \langle\nabla_{\theta} f_t(\theta), \nabla_{\theta} U_{\X_i}(\theta)\rangle m_t(\X_i|\theta) \rmd n_t(\theta) + \int_{\mathbb{R}^d} \Delta_{\theta} f_t(\theta) \rmd n_t(\theta) \nonumber \\
&=& \int_{\mathbb{R}^d} \Gt f_t(\theta) \rmd  n_t(\theta),
\end{eqnarray}
where we used the fact that $m_t(\theta,\X_i) = m_t(\X_i|\theta) n_t(\theta)$.
\end{proof}

\subsection{Moments upper bounds}

\begin{proposition}\label{prop:moment_U}
Assume $\Hn$, $\Hpi$, $\Hmin$ and that for each $\X_i$, $\theta \mapsto -\log p_\theta(\X_i)$ satisfies $\Hklsimple$. Then:
\begin{itemize}
\item[$i)$] Three positive constants $C_1$, $C_2$ and $C_3$, independent from $n$ and $d$, exist such that for any $t >0$:
$$ 
\mathbb{E}\left[e^{\frac{(1+r) n \c^{\frac{1}{1+r}}}{16} ( \|\theta_t\|_2^2 + 1 )^{\frac{1}{1+r}}}\right] \leq C_1
\left( \lnb \right)^{\frac{r}{1+r}} e^{ C_2 n \lnb }
 + C_3^d e^{\frac{(1+r)n \c^{\frac{1}{1+r}}}{16} }.
$$

\item[$ii)$] For any $t>0$ and for any $\alpha \ge 1$:
$$ \mathbb{E}[U_{\nu_n}^{\alpha}(\theta_t)] \led n^\alpha \left( \lnb \right)^{\alpha(1+r)}.
$$
\end{itemize}

\end{proposition}
\begin{proof}[Proof of $i)$] We consider the function $f(\theta) = \exp \left( \frac{a}{2} ( \|\theta\|_2^{2} + 1 )^{\rho} \right)$ where $0<\rho<1$, which is twice differentiable. The gradient of $f$ is computed as:
$$
\nabla f(\theta) = a \rho (\|\theta\|_2^2+1)^{\rho-1}   f(\theta) \theta.
$$
The Laplace operator is given as:
$$
\Delta f(\theta) = a \rho ( \|\theta\|_2^2 + 1 )^{\rho-2} f(\theta) \left[a \rho (\|\theta\|_2^2+1)^{\rho} \|\theta\|_2^2 + (d+2\rho-2) \|\theta\|_2^2 + d \right].
$$

We then deduce that for any $\theta \in \mathbb{R}^d$:
\begin{eqnarray}
\Gt f(\theta) & = & - \sum_{i=1}^n \langle \nabla U_{\X_i},\nabla f(\theta)  \rangle m_t(\X_i\vert \theta) + \Delta f(\theta) \nonumber\\
& = &  a \rho ( \|\theta\|_2^2 + 1 )^{\rho-2} f(\theta) \Big[ - ( \|\theta\|_2^2 + 1 ) \sum\limits_{i=1}^n \langle \theta, \nabla_{\theta} U_{\X_i}(\theta)\rangle m_t(\X_i|\theta) \nonumber \\
&& + a \rho (\|\theta\|_2^2+1)^{\rho} \|\theta\|_2^2 + \left(d+2\rho-2\right) \|\theta\|_2^2 + d \Big] \nonumber\\
& \leq &  a \rho ( \|\theta\|_2^2 + 1 )^{\rho-2} f(\theta) \Big[ - ( \|\theta\|_2^2 + 1 ) \sum_{i=1}^n \left(U_{\X_i}(\theta)-U_{\X_i}(0)\right) m_t(\X_i|\theta) \nonumber \\
&& + a \rho (\|\theta\|_2^2+1)^{\rho+1} + d \left( \|\theta\|_2^2 + 1\right) \Big] \nonumber \\
& \leq &  a \rho ( \|\theta\|_2^2 + 1 )^{\rho-1} f(\theta)  \left[ - \sum_{i=1}^n \left(U_{\X_i}(\theta)-U_{\X_i}(0)\right) m_t(\X_i|\theta) + a \rho (\|\theta\|_2^2+1)^{\rho} + d \right], \nonumber
\end{eqnarray}
where we used the convexity of $U_x$ for any position $x$.

Let us establish the bounds of $U_{\X_i}(\theta)$ and $U_{\X_i}(0)$. We denote by $\theta_i = \argmin U_{\X_i}$ and from Hypothesis $\Hmin$, there exist two positive constants $\mathcal{K}_1$ and $\mathcal{K}_2$ independent on $n$ and $d$ such that:
$$
\max_i \, \|\theta_i\|_2^2 \leq \mathcal{K}_1 \lnb \quad \mbox{ and } \quad \max_i \, U_{\X_i} (\theta_i) \leq \mathcal{K}_2 d. 
$$ 
We apply Proposition \ref{prop:growth_function_KL} to each non-negative function $U_{\X_i}$ that satisfies $\Hkl$, then we obtain that:
$$
U_{\X_i}(\theta)  \ge n \left[ \frac{(1+r)\c}{2} \right]^{\frac{1}{1+r}} \|\theta-\theta_i\|_2^{\frac{2}{1+r}}.
$$
Since $\frac{2}{1+r}>1$, the Jensen inequality yields $(u+v)^{\frac{2}{1+r}} \leq 2^{\frac{1-r}{1+r}} \left[ u^{\frac{2}{1+r}}+v^{\frac{2}{1+r}}\right]$, for all $(u,v) \in \mathbb{R}_+^2$ and we deduce that:
$$
\|\theta-\theta_i\|_2^{\frac{2}{1+r}} \ge 2^{\frac{r-1}{1+r}} \|\theta\|_2^{\frac{2}{1+r}} - \|\theta_i\|_2^{\frac{2}{1+r}} \ge 2^{\frac{r-1}{1+r}} \|\theta\|_2^{\frac{2}{1+r}} - \left(\mathcal{K}_1 \lnb  \right)^{\frac{1}{1+r}}.
$$
Then we use this inequality to obtain a lower bound of $U_{\X_i}$:
$$
U_{\X_i}(\theta) \ge 2n \left[ \frac{(1+r)\c}{8} \right]^{\frac{1}{1+r}} \|\theta\|_2^{\frac{2}{1+r}} - n \left[ \frac{(1+r)\c}{2} \right]^{\frac{1}{1+r}}(\mathcal{K}_1 \lnb )^{\frac{1}{1+r}} .
$$
Moreover an upper bound of $\max \, U_{\X_i}(0)$ comes from Proposition \ref{prop:kl_to_kl} and \ref{prop:growth_function_KL} as follows:
$$
U_{\X_i}(0) \leq U_{\X_i} (\theta_i) + \frac{n L + \Ll}{2} \|\theta_i\|_2^2 \leq \mathcal{K}_2 d + \frac{\mathcal{K}_1 (n L + \Ll) \lnb}{2}.
$$
Using the previous bounds and the fact that $\sum_{i=1}^n m_t(\X_i \vert \theta) =1 $, it yields:
\begin{eqnarray}
&& \sum_{i=1}^n \left(U_{\X_i}(\theta)-U_{\X_i}(0)\right) m_t(\X_i \vert\theta) \nonumber \\
&\ge & 2 n \left[ \frac{(1+r)\c}{8} \right]^{\frac{1}{1+r}} \|\theta\|_2^{\frac{2}{1+r}} - n \left[ \frac{(1+r)\c}{2} \right]^{\frac{1}{1+r}} (\mathcal{K}_1 \lnb )^{\frac{1}{1+r}} - \mathcal{K}_2 \Mnd - \frac{\mathcal{K}_1(n L + \Ll)\lnb}{2} \nonumber \\
& \ge & \frac{n \c^{\frac{1}{1+r}}}{4} \|\theta\|_2^{\frac{2}{1+r}} - n \c^{\frac{1}{1+r}} (\mathcal{K}_1 \lnb)^{\frac{1}{1+r}} - \mathcal{K}_2 \Mnd - \frac{\mathcal{K}_1 (n L + \Ll)\lnb }{2}, \nonumber
\end{eqnarray}
where we used some uniform upper bounds when $r \in [0,1)$.
We then choose $\rho=\frac{1}{1+r}$ and we deduce that: 
\begin{eqnarray}
\Gt f(\theta) & \le & \frac{a}{1+r} ( \|\theta\|_2^2 + 1 )^{-\frac{r}{1+r}} f(\theta) \left[ - \frac{n \c^{\frac{1}{1+r}}}{4} \|\theta\|_2^{\frac{2}{1+r}} + n \c^{\frac{1}{1+r}} (\mathcal{K}_1 \lnb)^{\frac{1}{1+r}} + \mathcal{K}_2 \Mnd \right. \nonumber \\
&& \left. + \frac{\mathcal{K}_1 (n L + \Ll)\lnb}{2} + \frac{a }{(1+r)} (\|\theta\|_2^2+1)^{\frac{1}{1+r}} + d \right] \nonumber \\
&\leq& \frac{a}{1+r} ( \|\theta\|_2^2 + 1 )^{-\frac{r}{1+r}} f(\theta)
\left[ - \left( \frac{n \c^{\frac{1}{1+r}}}{4} - \frac{a}{(1+r)} \right) \|\theta\|_2^{\frac{2}{1+r}} + n\c^{\frac{1}{1+r}} (\mathcal{K}_1 \lnb )^{\frac{1}{1+r}} \right. \nonumber \\
&& \left. + (\mathcal{K}_2 +1 ) \Mnd + \frac{\mathcal{K}_1 (n L + \Ll) \lnb }{2} + \frac{a}{(1+r)} \right], \nonumber
\end{eqnarray}
where we used $(\|\theta\|_2^2+1)^{\frac{1}{1+r}} \leq \|\theta\|_2^{\frac{2}{1+r}} + 1$ in the second line.

We now fix $a= \frac{n (1+r)\c^{\frac{1}{1+r}}}{8}$ and deduce that:
\begin{eqnarray}
\frac{\Gt f(\theta)}{f(\theta)} & \le & \frac{n^2 \c^{\frac{2}{1+r}}}{64}  ( \|\theta\|_2^2 + 1 )^{-\frac{r}{1+r}} \left[ - \|\theta\|_2^{\frac{2}{1+r}} +  8 (\mathcal{K}_1 \lnb )^{\frac{1}{1+r}} + \right. \nonumber \\
& & \left. + \frac{8(\mathcal{K}_2  + 1) \Mnd + 4 \mathcal{K}_1 (n L + \Ll)\lnb }{n \c^{\frac{1}{1+r}}} + 1 \right] .\label{eq:equation_technique_intermediaire}
\end{eqnarray}
We then study two complementary situations and below, we denote by $K_{n,d}$ the radius of the key compact set involved by the previous Lyapunov contraction: 
$$
K_{n,d}^{\frac{2}{1+r}} =   C \lnb.
$$

\noindent
$\bullet$ When $\|\theta\|_2$ is large enough ($\|\theta\|_2 \ge K_{n,d}$), we observe that a large enough $C>0$ independent from $n$ and $d$ exists such that:
\begin{equation}
\|\theta\|_2^{\frac{2}{1+r}} \ge  C  \lnb \Longrightarrow 
\frac{\Gt f(\theta)}{f(\theta)}  \le - \frac{n^2 \left(\lnb \right)^{\frac{1}{1+r}} \c^{\frac{2}{1+r}}}{128}  = - a_{n,d}.\label{eq:non_compact}
\end{equation}

\noindent
$\bullet$ When $\|\theta\|_2$ is upper bounded ($\|\theta\|_2 \le K_{n,d}$), we use the upper bound stated in Equation \eqref{eq:equation_technique_intermediaire}
and obtain that a universal $C_1$ (whose value may change from line to line) exists such that : 
\begin{align}
& \|\theta\|_2^{\frac{2}{1+r}} \le  C \lnb \Longrightarrow \nonumber \\
& \Gt f(\theta) \leq C_1 n^2 f(\theta) \left[  8 (\mathcal{K}_1 \lnb )^{\frac{1}{1+r}} + \frac{8(\mathcal{K}_2  + 1) \Mnd + 4 \mathcal{K}_1 (n L + \Ll)\lnb }{n \c^{\frac{1}{1+r}}} + 1\right] \nonumber \\
& \leq C_1 n^2 \lnb \exp \left( \frac{(C + 1) c^{\frac{1}{1+r}} n \lnb }{8}  \right) \leq  b_{n,d} e^{ \delta_{n,d} }.
\label{eq:compact}
\end{align}
We then use Equations \eqref{eq:non_compact} and \eqref{eq:compact} as follows. We define the function $\psi_{n,d}$ as $\psi_{n,d}(t) = \mathbb{E}[f(\theta_t)]$ and use Lemma \ref{lemma:derivation_semi_group}:
\begin{eqnarray*}
\psi'_{n,d}(t) & = & \mathbb{E}[ \Gt f(\theta_t) ] \\
& = & \mathbb{E}\left[ \Gt f(\theta_t) \left[\1_{\|\theta_t\|_2 \ge K_{n,d}}+ \1_{\|\theta_t\|_2 \le K_{n,d}}\right] \right] \\
& \leq  &  \mathbb{E}\left[ - a_{n,d} f(\theta_t) \1_{\|\theta_t\|_2 \ge K_{n,d}} + b_{n,d} e^{ \delta_{n,d} } \1_{\|\theta_t\|_2 \le K_{n,d}}  \right] \\
& \leq & - a_{n,d} \psi_{n,d}(t) +a_{n,d} \sup_{\|\theta\|_2\leq K_{n,d}} f(\theta)+b_{n,d} e^{ \delta_{n,d} } \\
& \leq & - a_{n,d} \psi_{n,d} (t) + (a_{n,d}+b_{n,d}) e^{ \delta_{n,d} }.
\end{eqnarray*}

We apply the Gronwall Lemma and obtain that:
\begin{equation}
\forall t > 0 \qquad \psi_{n,d}(t) \leq \left(1 + \frac{b_{n,d}}{a_{n,d}} \right) e^{ \delta_{n,d} } + \psi_{n,d}(0) e^{-a_{n,d}t}. \label{ineq:gronwall_lemma}    
\end{equation}

Using that $n_0$ is a Gaussian distribution, which was fixed in $\Hn$ hypothesis, we find an upper bound for $\psi_{n,d} (0) = E[f(\theta_0)] = \int_{\mathbb{R}^d} f(\theta) \rmd n_0(\theta)$ as follows :
\begin{eqnarray}
    \psi_{n,d} (0) &=& \left(2\pi \sigma^2 \right)^{-\frac{d}{2}} \int_{\mathbb{R}^d} e^{\frac{a}{2}\left(\|\theta\|_2^2+1\right)^\frac{1}{1+r} - \frac{\|\theta\|_2^2}{2\sigma^2}} \rmd \theta \nonumber \\
    &\leq& \left(2\pi \sigma^2 \right)^{-\frac{d}{2}} e^{\frac{a}{2}} \int_{\mathbb{R}^d} e^{ - \frac{\|\theta\|_2^2}{2} \left( \frac{1}{\sigma^2} - a \right) } \rmd \theta, \nonumber
\end{eqnarray}
if $\sigma^2 \leq \frac{1}{a} = \frac{8}{n(1+r)\c^\frac{1}{1+r}}$ then the integral above is finite. Since $c_2 < 1\leq \frac{8L}{(1+r)\c^{\frac{1}{1+r}}}$, it guarantees $\sigma^2 < \frac{1}{a}$, then:
\begin{eqnarray}
    \psi_{n,d} (0) &\leq& \left( 1- a \sigma^2 \right)^{-\frac{d}{2}} e^{\frac{a}{2}} \nonumber \\
    &\leq& C_3^d e^{\frac{(1+r) n \c^\frac{1}{1+r}}{16}}, \nonumber
\end{eqnarray}
where $C_3$ is a constant independent from $n$ and $d$.

Finally, using the value of $a_{n,d}$ and $b_{n,d}$ in \eqref{ineq:gronwall_lemma}, we deduce that:
$$
\mathbb{E}\left[e^{\frac{(1+r) n \c^{\frac{1}{1+r}}}{16} ( \|\theta_t\|_2^2 + 1 )^{\frac{1}{1+r}}}\right] \leq C_1
\left( \lnb \right)^{\frac{r}{1+r}} e^{ C_2 n \lnb }
 + C_3^d e^{\frac{(1+r)n \c^{\frac{1}{1+r}}}{16} }, \quad \forall t >0.
$$
where $C_2$ is another universal constant, which concludes the proof.
\end{proof}

\begin{proof}[Proof of $ii)$] We consider $\alpha>1$ and below, $C>0$ refers to a ``constant'' independent from $n$ and $d$, whose value may change from line to line.
Our starting point is the upper bound of the exponential moments obtained in $i)$. Proposition \ref{prop:kl_to_kl} shows that $U_{\nu_n}$ satisfies $\Hkl$, then thanks to Proposition \ref{prop:growth_function_KL}:
$$
\mathbb{E}[U_{\nu_n}^{\alpha}(\theta_t)] \leq \mathbb{E}\left[ \left(\min \, U_{\nu_n} + C n \|\theta_t - \theta^*_n \|_2^{2}\right)^{\alpha}
\right] \leq \mathbb{E}\left[ \left(\min \, U_{\nu_n} + C n \| \theta^*_n \|_2^{2} + C n \|\theta_t\|_2^2 \right)^{\alpha}
\right],
$$
where $\theta^*_n = \argmin U_{\nu_n}$. 

By using Proposition \ref{prop:growth_min_KL} and the inequality derived from the Jensen inequality $(a+b)^{\beta} \leq c_\beta (a^\beta+b^\beta)$ for $(a,b) \in \mathbb{R}_+^2$ and $\beta\geq 1$, we obtain that: 
\begin{align*}
\left(\min \, U_{\nu_n} + \right.& \left. C n \|\theta^*_n\|_2^{2} + C n \|\theta_t\|_2^2 \right)^{\alpha} \\
&\leq  C \left[ \MndU + n d^{1+r} \log^{2\beta(1+r)} (n) + n \|\theta_t\|_2^{2} \right]^{\alpha} \\
&\leq  C n^\alpha \left[ \left( \lnb \right)^{\alpha (1+r)} + \|\theta_t \|_2^{ 2   \alpha }\right] \\
& \leq   C n^\alpha \left[ \left( \lnb \right)^{\alpha (1+r)} +  k^{-\alpha(1+r)} \log^{\alpha(1+r)} \left( e^{k  \|\theta_t \|_2^{ \frac{2}{1+r}}} \right)  \right] \\
&\leq  C n^\alpha \left[ \left( \lnb \right)^{\alpha (1+r)} + k^{-\alpha(1+r)} \log^{\alpha(1+r)} \left( e^{\alpha(1+r)-1+k  \|\theta_t\|_2^{ \frac{2}{1+r}}} \right) \right] .
\end{align*}
The Jensen inequality and the concavity of $x \mapsto \log^{p}(x)$ on $[e^{p-1},+\infty[$ when $p\geq 1$ yield
\begin{align*}
\mathbb{E}[&U_{\nu_n}^{\alpha}(\theta_t)] \\
& \leq C n^\alpha \left[ \left( \lnb \right)^{\alpha (1+r)} + k^{-\alpha(1+r)} \mathbb{E}\left[  \log^{\alpha(1+r)} \left( e^{\alpha(1+r)-1+k  \|\theta_t\|_2^{ \frac{2}{1+r}}} \right) \right] \right] \\
& \leq  C n^\alpha \left[ \left( \lnb \right)^{\alpha (1+r)} + k^{-\alpha(1+r)} \log^{\alpha(1+r)} \left[ \mathbb{E} \left( e^{\alpha(1+r)-1+ k  \|\theta_t\|_2^\frac{2}{1+r}} \right) \right] \right] \\
& \leq  C n^\alpha \left[ \left( \lnb \right)^{\alpha (1+r)} + k^{-\alpha(1+r)} \left[ \alpha(1+r)-1 + \log \mathbb{E}\left( e^{  k  \|\theta_t\|_2^\frac{2}{1+r}} \right) \right]^{\alpha(1+r)} \right]\\
&\leq C n^\alpha \left[ \left( \lnb \right)^{\alpha (1+r)} + k^{-\alpha(1+r)} \left[ \alpha(1+r)-1 + \log  \mathbb{E}\left( e^{  k (\|\theta_t \|_2^2 + 1 )^{\frac{1}{1+r}} } \right) \right]^{\alpha(1+r)} \right],
\end{align*}
where we used in the last inequality that $\|\theta\|_2^2 \leq \|\theta\|_2^2 + 1$.

We then apply $i)$ in Proposition \ref{prop:moment_U}, we choose $k = \frac{(1+r) n \c^{\frac{1}{1+r}}}{16} $ and obtain that: 
\begin{align*}
\mathbb{E}[U_{\nu_n}^{\alpha}(\theta_t)] &\\ 
& \leq C n^\alpha \left[ \left( \lnb \right)^{\alpha (1+r)} + \frac{1}{n^{\alpha (1+r)}} \left[ 1 + \log  \mathbb{E}\left( e^{ \frac{(1+r)n \c^{\frac{1}{1+r}}}{16} (\|\theta_t \|_2^2 + 1 )^{\frac{1}{1+r}} } \right) \right]^{\alpha(1+r)} \right]\\
& \leq  C \left[ n^\alpha \left( \lnb \right)^{\alpha (1+r)} \right. \\
&  \left. + \frac{1}{n^{\alpha r}} \left[ 1 +  \log \left[ C_1 \left( \lnb \right)^{\frac{r}{1+r}} e^{ C_2 n \lnb } + C_3^d e^{\frac{(1+r) n \c^{\frac{1}{1+r}}}{16}} \right] \right]^{\alpha (1+r)} \right] \\
& \leq  C n^\alpha \left( \lnb \right)^{\alpha(1+r)},
\end{align*}
where we used in the previous lines simple algebra and $\log(a+b) \leq \log(2)  + \log(a)+\log(b)$ when $a \geq 1$ and $b\geq 1$. This concludes the proof. 
\end{proof}
 






 \end{document}